\documentclass{article}

\usepackage{microtype}
\usepackage{graphicx}
\usepackage{subfigure}
\usepackage{booktabs} 
\usepackage{makecell}
\usepackage[table]{xcolor}
\usepackage{hyperref}
\usepackage{booktabs}

\usepackage[accepted]{icml2024}

\usepackage{amsmath}
\usepackage{amssymb}
\usepackage{mathtools}
\usepackage{amsthm}

\usepackage{wrapfig,lipsum,booktabs}

\usepackage[capitalize,noabbrev]{cleveref}

\theoremstyle{plain}
\newtheorem{theorem}{Theorem}[section]

\theoremstyle{definition}
\newtheorem{definition}[theorem]{Definition}

\theoremstyle{remark}

\newcommand{\padspace}{\hspace{3.5em}}
\DeclareDocumentCommand{\p}{ D<>{p} D<>{} r() }{
  \def\content{#3}\patchcmd{\content}{|}{\;#2\vert\;}{}{}
  \ensuremath{#1 #2(\content #2)}}
\DeclareDocumentCommand{\pp}{ D<>{} r() }{
\ensuremath{\p<p_\phi><#1>(#2)}}
\DeclareDocumentCommand{\qp}{ D<>{} r() }{
\ensuremath{\p<q_\phi><#1>(#2)}}

\usepackage[textsize=tiny]{todonotes}
\usepackage{csquotes}
\icmltitlerunning{Learning Latent Dynamic Robust Representations for World Models}

\begin{document}

\twocolumn[
\icmltitle{Learning Latent Dynamic Robust Representations for World Models}



\icmlsetsymbol{equal}{*}
\icmlsetsymbol{correspond}{$^\dagger$}

\begin{icmlauthorlist}
\icmlauthor{Ruixiang Sun}{equal,yyy}
\icmlauthor{Hongyu Zang}{equal,yyy}
\icmlauthor{Xin Li}{correspond,yyy}
\icmlauthor{Riashat Islam}{sch}
\end{icmlauthorlist}

\icmlaffiliation{yyy}{Beijing Institute of Technology, China}
\icmlaffiliation{sch}{DreamFold AI, Canada}

\icmlcorrespondingauthor{Xin Li}{xinli@bit.edu.cn}

\icmlkeywords{Machine Learning, ICML}

\vskip 0.3in
]



\printAffiliationsAndNotice{\icmlEqualContribution} 

\begin{abstract}
Visual Model-Based Reinforcement Learning (MBRL) promises to encapsulate agent's knowledge about the underlying dynamics of the environment, enabling learning a world model as a useful planner. However, top MBRL agents such as Dreamer often struggle with visual pixel-based inputs in the presence of exogenous or irrelevant noise in the observation space, due to failure to capture task-specific features while filtering out irrelevant spatio-temporal details. To tackle this problem,  we apply a spatio-temporal masking strategy, a bisimulation principle, combined with latent reconstruction, to capture endogenous task-specific aspects of the environment for world models, effectively eliminating non-essential information. Joint training of representations, dynamics, and policy often leads to instabilities. To further address this issue, we develop a Hybrid Recurrent State-Space Model (HRSSM) structure, enhancing state representation robustness for effective policy learning. Our empirical evaluation demonstrates significant performance improvements over existing methods in a range of visually complex control tasks such as Maniskill~\cite{gu2023maniskill2} with exogenous distractors from the Matterport environment. Our code is avaliable at \url{https://github.com/bit1029public/HRSSM}.




\end{abstract}
\vspace{-4mm}
\section{Introduction}



Model-Based Reinforcement Leanring (MBRL) utilizes predictive models to capture endogenous dynamics of the world, to be able to simulate and forecast future scenarios, enhancing the agent's decision making abilities by leveraging imagination and prediction in visual pixel-based contexts ~\cite{hafner2019dream, kalweit2017uncertainty, hafner2019learning, ha2018recurrent, janner2020gamma}. Most importantly, these world models such as recurrent state-space model (RSSM)~\cite{hafner2019learning}, enable agents to understand and represent dynamics in the learnt representation space, consisting of task specific information with the hope to have filtered out exogenous or irrelevant aspects from the observations, leading to superior performance compared to model-free RL algorithms. However, most MBRL methods face challenges in environments with large amounts of unpredictable or irrelevant exogenous observations~\cite{burda2018exploration, efroni2022sample, efroni2021provably}.

\begin{figure}[t]
\centering
\includegraphics[width=0.95\linewidth]{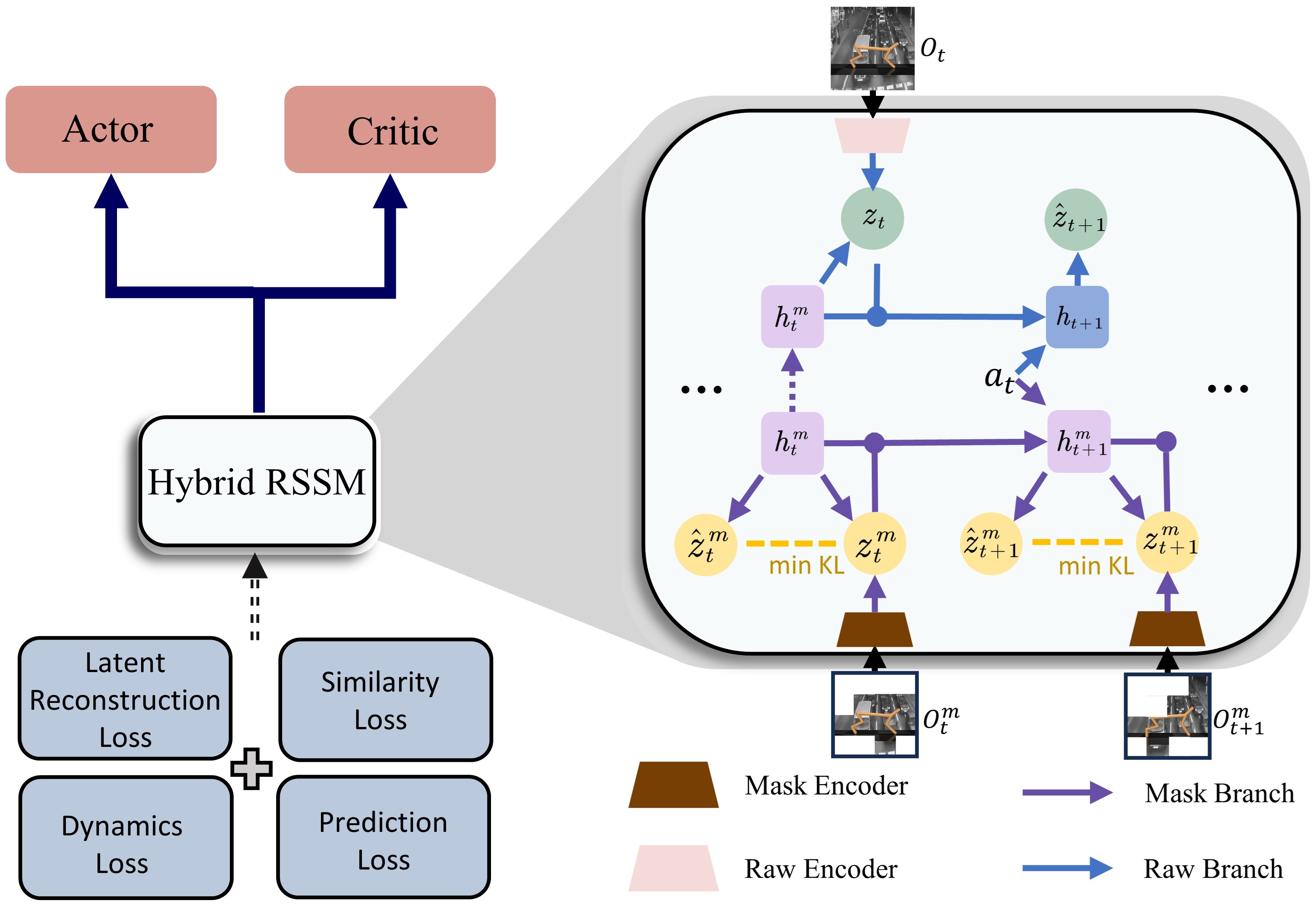}
    \caption{Our framework is composed of a Hybrid-RSSM and actor-critic architecture. Hybrid-RSSM learns robust representations and dynamics through four distinct objectives: latent reconstruction, which aligns features between masked and raw observations; similarity loss based on the bisimulation principle; and two additional objectives same as in Dreamer series~\cite{hafner2020mastering,hafner2023mastering}. \vspace{-5mm} }
\label{fig:model}
\end{figure}

Arguably, the Dreamer series of algorithms ~\cite{hafner2019dream,hafner2020mastering,hafner2023mastering} are probably the most effective and representative class of MBRL approaches 
where agents learn representations and dynamics in latent space by minimizing reconstruction errors. Most MBRL approaches such as Dreamer often includes a forward dynamics model to predict observations and a reward model that evaluates the potential of future states. Recent works however have shown the ineffectiveness of forward dynamics based models when learning from exogenous observation based visual inputs \cite{efroni2022sample, lamb2022guaranteed, islam2022agent}. This is because in noisy environments, emphasis on reconstruction can lead to disproportionate focus on irrelevant details such as textures or noise, at the expense of smaller but task-relevant elements. This can result in inaccuracies in the dynamics model~\cite{xiao2019learning, asadi2019combating} and overfitting to specific environmental traits~\cite{zhang2020invariant}, leading to compounded errors in latent space world models for planning. 

While a body of work has addressed exogenous noise, primarily in reward-free~\cite{efroni2021provably, efroni2022sample, lamb2022guaranteed} or offline visual settings for model-free RL~\cite{islam2022agent}, only limited research has explored model-based agents in the context of exogenous noise. These studies have developed in a way of decoder-free matter, \textit{i.e.}, excluding pixel-level reconstruction, to mitigate reconstruction issues, but they still face significant challenges: either lacking in capturing task-specific information~\cite{deng2022dreamerpro,okada2021dreaming}, not being robust against various noise types~\cite{fu2021learning}, or sensitive to hyper-parameters~\cite{DBLP:journals/corr/abs-2309-00082, islamrepro}. This work is therefore primarily driven the question :

\textit{How to learn sufficiently expressive state representation for a world model without the reliance of the pixel-level reconstruction?}

In principle, the ideal representation objective for model-based planners should address two desired criterion :  i) effectively capturing task-relevant endogenous dynamics information, and ii) be robust and compact enough to filter exogenous task irrelevant details. Despite several prior works trying to address this \cite{lamb2022guaranteed, efroni2021provably, islam2022agent} in reward free settings, these works do not show effectiveness of the learnt representation for use in world models. We address this question through the promising approach of bisimulation principle~\cite{ferns2011bisimulation,castro2020scalable, zhang2020learning, castro2021mico, zang2022simsr}, learning representations specific to task objectives that can reflect state behavioral similarities. However, the effectiveness of the bisimulation metric heavily relies on the accuracy of the dynamics model~\cite{kemertas2021towards}. Under an approximate dynamics model, the state representation guided by the bisimulation principle may be task-specific but not necessarily compact, indicating a gap in the bisimulation principle's ability to foster expressive state representations for robust model-based agents.


To effectively apply bisimulation principle in world models, we propose to develop a new architecture - the Hybrid-RSSM (HRSSM).
 This architecture employs a masking strategy to foster more compact latent representations, specifically targeting the integration of the bisimulation principle to improve the efficiency and effectiveness of the model.  Our Hybrid-RSSM consists of two branches: 1) the raw branch, which processes original interaction sequences, and 2) the mask branch, which handles sequences that have been transformed using a masking strategy. This masking, involving cubic sampling of observation sequences, is designed to reduce spatio-temporal redundancy in natural signals. A key feature of our approach is the reconstruction of masked observations to match the latent features from the raw branch in the latent representation space, not in pixel space. This ensures semantic  alignment for both branches. Meanwhile, we incorporate a similarity-based objective, in line with the bisimulation principle, to integrate differences in immediate rewards and dynamics into the state representations. 
 
 Furthermore, to enhance training stability and minimize potential representation drift, the raw and mask branches share a unified historical information representation. This holistic structure defines our Hybrid Recurrent State Space Model (HRSSM), serving as a world model that leverages the strengths of the RSSM architecture to effectively capture task-specific information, guided by the bisimulation principle, and efficiently condense features through mask-based latent reconstruction. 
Our primary contributions are summarized as follows.
\begin{itemize}
    \item We introduce Hybrid RSSM that integrates masking-based latent reconstruction and the bisimulation principle into a model-based RL framework, enabling the learning of task-relevant representations capturing endogenous dynamics.
    \item We study the roles of masking-based latent reconstruction and the bisimulation principle in model-based RL with empirical and theoretical analysis.

    \item Empirically, we evaluate our Hybrid-RSSM and actor-critic architecture by integrating it into the DreamerV3 framework, and show that the resulting model can be used to solve complex tasks consisting of a variety of exogenous visual information. 
\end{itemize}


\section{Related Work}
\textbf{MBRL and World Model}
Model-based Reinforcement Learning (MBRL) stands as a prominent subfield in Reinforcement Learning, 
aiming to optimize total reward through action sequences derived from dynamics and reward models~\cite{sutton1990integrated, hamrick2019analogues}. Early approaches in MBRL typically focus on low-dimensional and compact state spaces~\cite{williams2017model, janner2019trust, janner2020gamma}, yet they demonstrated limited adaptability to more complex high-dimensional spaces. 
Recent efforts~\cite{hafner2019learning, hafner2019dream, hafner2020mastering, hafner2023mastering, hansen2022modem, rafailov2021offline, gelada2019deepmdp} have shifted towards learning world models for these intricate spaces, utilizing visual inputs and other signals like scalar rewards. These methods enable agents to simulate behaviors in a conceptual model, thereby reducing the reliance on  physical environment interactions. As a notable example, Dreamer~\cite{hafner2019dream, hafner2020mastering,hafner2023mastering} learns recurrent state-space models (RSSM) and the latent state space via reconstruction losses, though achieving a good performance in conventional environments yet fails in environments with much exogenous noise.

\textbf{Model-based Representation Learning}
Many recent MBRL methods start to integrate state representation learning into their framework to improve the robustness and efficiency of the model.
Some approaches formulations rely on strong assumptions~\cite{gelada2019deepmdp, agarwal2020flambe}. Some approches learn world model via requiring latent temporal consistency~\cite{zhao2023simplified, hansen2022temporal, hansen2023td}. Some approaches develop upon Dreamer architecture, combining the transformer-based masked auto-encoder~\cite{seo2023masked}, extending Dreamer by explicitly modeling two independent latent MDPs that represent useful signal and noise, respectively~\cite{fu2021learning, wang2022denoised}, optimizing the world model by utilizing mutual information~\cite{DBLP:journals/corr/abs-2309-00082}, regularizing world model via contrastive learning~\cite{okada2021dreaming, poudel2023recore} and prototype-based representation learning~\cite{deng2022dreamerpro}. Unlike other approaches that either neglect reward significance or are limited by modeling predefined noise form, our approach learns robust representations and dynamics effectively by incorporating reward-aware information and masking strategy, we provide a more detailed comparison and additional related works in Appendix~\ref{app:comparison}.

\section{Preliminaries}
\label{sec:prelim}
\textbf{MDP}
The standard Markov decision process (MDP) framework is given by a tuple $\mathcal{M}=(\mathcal{S},\mathcal{A}, P, r, \gamma)$, with state space $\mathcal{S}$, action space $\mathcal{A}$, reward function $r(s,a)$ bounded by $[R_\text{min},R_\text{max}]$, a discount factor $\gamma\in [0,1)$, and a transition function $P(\cdot,\cdot):\mathcal{S}\times\mathcal{A}\rightarrow\Delta\mathcal{S}$ that decides the next state, where the transition function can be either deterministic, i.e., $s'= P(s,a)$, or stochastic, i.e. $s'\sim P(\cdot|s,a)$. In the sequel, we use $P_{s}^{a}$ to denote $P(\cdot|s, a)$ or $P(s,a)$ for simplicity.
The agent in the state $s\in \mathcal{S}$ selects an action $a\in \mathcal{A}$ according to its policy, mapping states to a probability distribution on actions: $a\sim\pi(\cdot|s)$. 
 We make use of the state value function $V^{\pi}(s)=\mathbb{E}_{\mathcal{M}, \pi}\left[\sum_{t=0}^{\infty} \gamma^{t} r\left(s_{t}, a_{t}\right) \mid s_{0}=s\right]$  to describe the long term discounted reward of policy $\pi$ starting at the state $s$, where $\mathbb{E}_{\mathcal{M},\pi}$ denotes expectations under $s_0 \sim P_0$, $a_t\sim \pi(\cdot|s_t)$, and $s_{t+1}\sim P_s^a$. And the goal is to learn a policy $\pi$ that maximizes the sum of expected returns $\mathbb{E}_{\mathcal{M}, \pi}\left[\sum_{t=0}^{\infty} \gamma^{t} r\left(s_{t}, a_{t}\right) \mid s_{0}=s\right]$.

\textbf{Visual RL and Exogenous noise}
We address visual reinforcement learning (RL) where the agent perceives high-dimensional pixel images as observations, represented by $o_t\sim P(o_t|o_{<t},a_{<t})$. These observations are mapped into a lower-dimensional space via a transformation $\mathcal{T}$ and an encoder $\mathcal{E}$, \textit{i.e.}, $\mathcal{T} \circ\mathcal{E}:\mathcal{O}\rightarrow\mathcal{X}$, then generating a latent state in a latent space: $\zeta_t\in\mathcal{Z}$ through a world model. The agent's actions follow a policy distribution $\pi(a|\zeta)$ under this latent state space. We introduce a setting with exogenous noise, where observations come from a mix of controllable endogenous states $s_t\in \mathcal{S}$ and uncontrollable exogenous noise $\xi_t \in\Xi$. Here, $\zeta_t$ is composed of these two components, with transitions $P(\zeta_t|\zeta_{<t},a_{<t})=P(s_t|s_{<t},a_{<t})P(\xi_t|\xi_{<t})$, and rewards $r(\zeta_t,a_t)=r(s_t,a_t)$. We strive to compress latent state $\zeta_t$ by maximizing endogenous state $s_t$ and minimizing exogenous noise $\xi_t$, deriving an ``exogenous-free'' policy, essentially $\pi(a|\zeta)\approx \pi(a|s)$. Under a mild assumption of existing mapping function $\phi_\star$ from the observation $o\in\mathcal{O}$ to the endogenous state $s\in\mathcal{S}$, for any $o_1$ and $o_2$, if $\phi_\star(o_1)=\phi_\star(o_2)$, then $\pi(\cdot|o_1)=\pi(\cdot|o_2)$. The primary goal is to learn a world model that can discard exogenous noise and learn exo-free policy to improve the sample efficiency and robustness.

\section{Method}
In this section, we describe our overall approach of integrating the masking strategy and bisimulation principle in model-based RL methods, to learn effective world models for planning. We show that out method can be adapted to learn effective representatons in the presence of exogenous noise, and the resulting planner can be used to solve complex tasks, building on the DreamerV3~\cite{hafner2023mastering}. The whole pipeline is shown in Figure~\ref{fig:pipeline}.

\begin{figure}[t]
\centering
\includegraphics[width=0.95\linewidth]{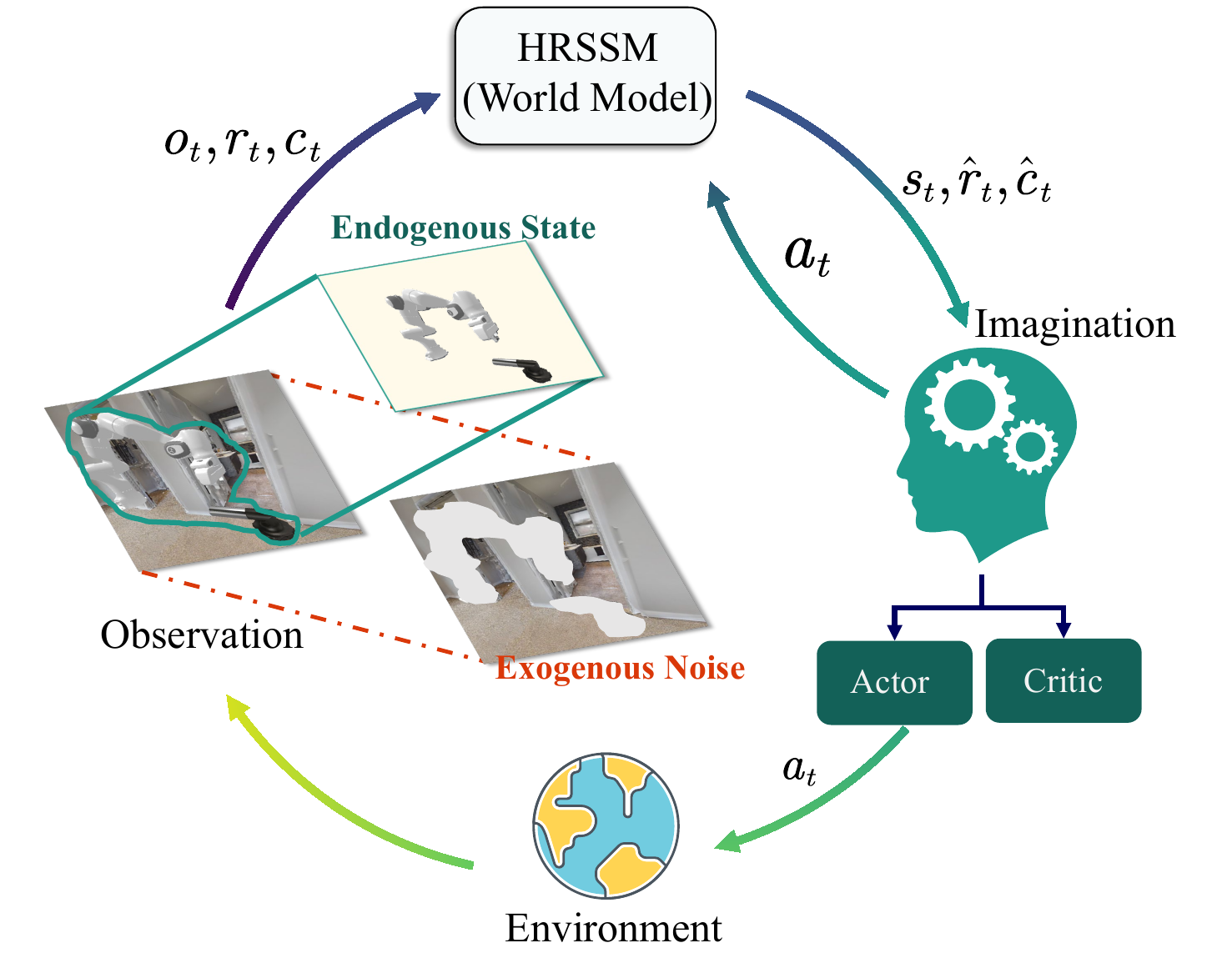}
    \caption{The entire pipeline of our framework in the presence of exogenous information. The HRSSM processes the observations into the latent space, enabling the agent to learn control within this space. Subsequently, the policy network generates actions for interacting with the environment. \vspace{-5mm} }
\label{fig:pipeline}
\end{figure}



\textbf{Modified Dreamer Architecture}
Dreamer utilizes a recurrent state space model (RSSM) \cite{hafner2019learning} for differentiable dynamics, learning representations of sensory inputs through backpropagation of Bellman errors from imagined trajectories. Its training process involves: optimizing the RSSM, training a policy using latent imaginations, and applying this policy in the real environment. This cycle repeats until the desired policy performance is achieved. The RSSM includes several crucial components: 
\begin{equation}
\label{eq:rssm}
\begin{alignedat}{4}
& \text{Recurrent model:} \padspace && h_t &\ = &\ f_\phi(h_{t-1},z_{t-1},a_{t-1}) \\
& \text{Representation model:} \padspace && z_t &\ \sim &\ \qp(z_t | h_t,o_t) \\
& \text{Transition predictor:} \padspace && \hat{z}_t &\ \sim &\ \pp(\hat{z}_t | h_t) \\
& \text{Reward predictor:} \padspace && \hat{r}_t &\ \sim &\ \pp(\hat{r}_t | h_t,z_t) \\
& \text{Continue predictor:} \padspace && \hat{c}_t &\ \sim &\ \pp(\hat{c}_t | h_t,z_t) \\
& \text{Decoder:} \padspace && \hat{o}_t &\ \sim &\ \pp(\hat{o}_t | h_t,z_t),
\end{alignedat}
\end{equation}
where $o_t$ is the sensory input, $z_t$ the stochastic representation, $h_t$ the recurrent state, $\hat{o}_t$ the reconstructed input, and $\hat{r}_t$ and $\hat{c}_t$ are the predicted reward and the episode continuation flag. 
While the decoder network is crucial in Dreamer for learning environment dynamics, its reliance on reconstructing high-dimensional sensory inputs like pixels causes computational inefficiency, which arises from recovering unnecessary, control-irrelevant visual elements such as background noise, impeding policy learning in environments with distractions. Prior works have explored how to recover the full endogenous latent states, by ignoring exogenous noise \cite{islam2022agent}; however, effectively recovering endogenous dynamics for model-based planning remains unaddressed. We aim to develop a method for recovering these dynamics for model-based planning. Simply omitting pixel reconstruction from Dreamer, as suggested by \cite{hafner2019dream}, results in inadequate performance. Therefore, we propose modifying Dreamer to preserve accurate dynamics and enhance its awareness of essential downstream task features, while reducing dependency on reconstruction.

\subsection{Learning latent representation and dynamics}
In visual control tasks, our state representation concentrates on two key aspects: (i) visual inputs includes much spatio-temporal redundancy, and (ii) the encapsulation of behaviorally relevant information for the task. We introduce two novel components: masking-based latent reconstruction and similarity-based representation. The former filters out redundant spatiotemporal data while preserving semantic useful environmental knowledge. The latter, aligning with the bisimulation principle, retains task-specific information within the world model. This approach results in latent representations that are concise and effective. 

Notably, our method may not recover the full endogenous dynamics, but can still be exo-free, distinguishing from other works \cite{lamb2022guaranteed}. Our key contribution is demonstrating adaptability to MBRL methods for planning, an area not fully addressed by prior research. We include detailed analysis of our proposed methodology in section ~\ref{sec:analysis}.
To keep the notation succinct, we will replace $\zeta$ with $s$ since our goal is to disregard $\xi$ and we will ensure to remind readers of this when necessary.

\paragraph{Masking strategy} Our goal is to design world models for planning that can be effective in the presence of visual exogenous information. To do this, we employ a masking strategy to reduce the spatio-temporal redundancy for enhanced control task representations. 
In visual RL tasks, previous works~\cite{tong2022videomae, wei2022masked} indicate that significant spatio-temporal redundancy can be removed via masking based reconstruction methods.  Consequently, we randomly mask a portion of pixels in the input observation sequence across its spatial and temporal dimensions. For a series of $K$ environmental interaction samples $\{o_t,a_t,r_t\}_{t=1}^K$, we transform the observation sequence $\mathbf{o}=\{o_t\}_{t=1}^K\in \mathbb{R}^{K\times H\times W\times C}$ into cuboid patches $\mathbf{\hat{o}}=\{\hat{o}_t\}_{t=1}^K\in\mathbb{R}^{kP_K\times hP_H\times wP_W \times C}$, where the patch size is $(P_K\times P_H \times P_W)$ and $k=K/P_K$, $h=H/P_H$, $w=W/P_W$ are the number of patches along each dimension. We then randomly mask a fraction $m$ of these cuboid patches to capture the most essential spatio-temporal information while discarding spatio-temporal redundancies. Subsequently, both the masked and original sequences are encoded to latent encoding space using an encoder and a momentum encoder respectively, where the momentum encoder is updated using an exponential moving average (EMA) from the masked sequence's encoder. 
\begin{table*}[htbp]
\centering
\caption{Model components of our hybrid structure. EMA means the corresponding model is updated via exponential moving average. Gradient back-propagates through mask models and reward/continue predictor.}
\begin{tabular}{ll} \hline\hline
Mask Encoder: $e_t^m = \mathcal{E}_\phi(o_t^m)$ & EMA Encoder: $e_t = \mathcal{E}'_\phi(o_t)$            \\
Mask Posterior model: $z^m_t \sim \qp(z^m_t | h^m_t,e^m_t)$             & EMA Posterior model: $z_t \sim q'_\phi(z_t | h^m_t,e_t)$             \\
Mask Recurrent model: $h^m_t = f_\phi(h^m_{t-1},z^m_{t-1},a_{t-1})$         & EMA Recurrent model: $h_t = f'_\phi(h^m_{t-1},z_{t-1},a_{t-1})$           \\
Mask Transition predictor: $\hat{z}^m_t \sim \pp(\hat{z}^m_t | h^m_t)$      & EMA Transition predictor: $\hat{z}_t \sim p'_\phi(\hat{z}_t | h_t)$ \\ \hline
Reward predictor: $\hat{r}_t \sim \pp(\hat{r}_t | h^m_t,z^m_t)$    &   Continue predictor: $\hat{c}_t \sim \pp(\hat{c}_t | h^m_t,z^m_t)$  \\ \hline\hline
\end{tabular}
\label{tab:hybrid_rssm}
\end{table*}

\vspace{-5mm}\paragraph{Behavioral update operator} To capture the task relevant information for control tasks, we adopt a similarity-based objective following the bisimulation principle~\cite{ferns2012metrics, ferns2012methods}, which requires the learnt representation to be aware of the reward and dynamics similarity between states. Our mask-based behavioral update operator, for masked and original sequences can be written as : 
\begin{equation}
\begin{aligned}
    \mathcal{F}^\pi d(s_i,s_j^m) =|r_{s_i}^\pi-r_{s_j}^\pi|+\gamma \mathbb{E}_{\begin{subarray} {l}s_{i+1}\sim \hat{P}_{s_{i}}^\pi, \\s_{j+1}^m\sim \hat{P}_{s_{j}^m}^\pi\end{subarray}}[d(s_{i+1},s_{j+1}^{m})],
\end{aligned}
\label{eq:bisim}
\end{equation}
where $s_j^m$ and $s_i$ represent latent states of the mask branch and the raw branch, respectively, with $\hat{P}_{s_j^m}^\pi$ and $\hat{P}_{s_i}^\pi$ denoting their approximated latent dynamics, and $d$ is the cosine distance to measure the difference between latent states. 

We use equation ~\ref{eq:bisim} to minimize bisimulation error for learning representation. However, this process involves sampling from latent dynamics, which, when coupled with the simultaneous learning of representations, dynamics, and policies in the world model, can lead to instabilities that adversely impact dynamics learning and consequently, bisimulation training. Therefore, we develop a hybrid RSSM specifically to address complex tasks, providing a level of stability in MBRL methods, which otherwise is typically difficult due to the complexities associated with training joint objectives.

\paragraph{Hybrid RSSM}
We first follow the conventional setting of RSSM in DreamerV3 to build in the masked encoding space, \textit{i.e.}, a mask encoder $e_t^m=\mathcal{E}_\phi(o_t^m)$ to encode the masked observation, a mask posterior model $z^m_t \sim \qp(z^m_t | h^m_t,e^m_t)$ and a mask recurrent model $h^m_t = f_\phi(h^m_{t-1},z^m_{t-1},a_{t-1})$ to incorporate temporal information into representations, and a mask transition predictor $\hat{z}^m_t \sim \pp(\hat{z}^m_t | h^m_t)$ to model the latent dynamics, where the concatenation of the mask recurrent state $h_t^m$ and the mask posterior state $z_t^m$ forms the mask latent state $s_t^m:=[h_t^m;z_t^m]$. We train the dynamics model by minimizing the KL divergence between the posterior state $z^m_t$ and the predicted prior state $\hat{z}^m_{t}$, and employ free bits~\cite{kingma2016improved, hafner2023mastering}, formulated as:
\begin{equation}
  \begin{aligned}
    \mathcal{L}_{\mathrm{dyn}}(\phi)&:=\beta_{\mathrm{1}}\mathrm{max}(1, \mathcal{L}_{\mathrm{1}}(\phi))+\beta_{\mathrm{2}}\mathrm{max}(1, \mathcal{L}_{\mathrm{2}}(\phi)) \\
    \mathcal{L_{\mathrm{1}}}(\phi)&:=\operatorname{KL}\!\big{[}\operatorname{sg}(q_{\phi}(z^m_{t}\;|\;h^m_{t},e^m_{t}))\;\big{\|}\;p_{\phi}(\hat{z}^m_{t}\;|\;h^m_{t})\hphantom{)}\big{]}\\
\mathcal{L_{\mathrm{2}}}(\phi)&:=\operatorname{KL}\!\big{[}q_{\phi}(z^m_{t}\;|\;h^m_{t},e^m_{t})\hphantom{)}\;\big{\|}\;\operatorname{sg}(p_{\phi}(\hat{z}^m_{t}\;|\;h^m_{t}))\big{]}
  \label{eq:KL loss}
  \end{aligned}
  \end{equation}
where $\operatorname{sg}$ means stopping gradient, and the values of $\beta_1$ and $\beta_2$ are set to 0.5 and 0.1, respectively, following the default configuration in DreamerV3.
For now, we only construct the network of the masked sequence, but without the utilization of the original sequence. If the raw branch utilizes a different RSSM structure from the mask one, merging these complex networks could lead to training instability and representation drift. To address this, we require the raw branch and the mask branch share the same historical representation, ensuring alignment between both branches for temporal prediction. Therefore, for the raw branch, we conduct the posterior state as $z_t\sim q'_\phi(z_t | h^m_t,e_t)$, the recurrent state $h_t=f'_\phi(h^m_{t-1},z_{t-1},a_{t-1})$ with the historical representation from the mask branch, and the prior state $\hat{z}_t\sim p'_\phi(\hat{z}_t|h_t)$. Additionally, we define the latent state of raw branch as $s_t:=[h_t^m;z_t]$ and the sampled latent state of RSSM as $\hat{s}_t=[h_t;\hat{z}_t]$. The networks $q'_\phi$, $f'_\phi$, and $p'_\phi$ are all updated using EMA from the mask branch.

We use latent reconstruction to align the feature between the masked and original ones, to disregard the unnecessary spatiotemporal redundancies, following the research within the field of computer vision~\cite{he2022masked, videomae} that considering high-dimensional image space consists tramendous spatiotemporal redundancies. We apply a linear projection and $\ell_2$-normalize the latent state $s_t$ and $s_t^m$ to obtain $\bar{s}_t$ and $\bar{s}^m_t$ respectively to ensure numerical stability and then compute the reconstruction loss, which can be formulated as:
\begin{equation}
  \begin{aligned}
    \mathcal{L_{\mathrm{rec}}}(\phi)&:= \mathrm{MSE}(\bar{s}_t, \bar{s}^m_t).
  \label{eq:rec_loss}
  \end{aligned}
\end{equation} 
Meanwhile, we can minimize the bisimulation error and formulate the similarity loss to capture the task-relevant information as:
\begin{equation}
\begin{aligned}
\mathcal{L}_\text{sim}&:=\left(d(s_i,s_j^m)-\mathcal{F}^\pi d(s_i,s_j^m)\right)^2\\&=\left(d(s_i,s_j^m)-\left(|r_{s_i}^\pi-r_{s_j}^\pi|+\gamma d(\hat{s}_{i+1},\hat{s}_{j+1}^{m})\right)\right)^2,
\label{eq:beh_loss}
\end{aligned}
\end{equation}
where $d$ is the cosine distance, $\hat{s}_{i+1}$ and $\hat{s}_{j+1}^m$ are sampled from RSSMs.

\paragraph{Reward Prediction and Continue Prediction}
Following DreamerV3~\cite{hafner2023mastering}, we train the reward predictor via the symlog loss and the continue predictor via binary classification loss, to predict the reward and the episode is termination or not, they compose the prediction loss as:
\begin{equation}
    \mathcal{L_{\mathrm{pred}}}(\phi):=-\ln p_{\phi}(r_{t}\;|\;s_t^m)-\ln p_{\phi}(c_{t}\;|\;s_t^m).
    \label{eq:pred_loss}
\end{equation}

Gradient backpropagation occurs exclusively through the mask branch, updating the representation. Consequently, we utilize only the masked latent state $s_t^m$ for predicting both terms. Unlike the methods in Equations~\ref{eq:rec_loss} and~\ref{eq:beh_loss}, we employ un-normalized features for prediction. Empirically, this approach enhances the model's stability and sample-efficiency, as detailed in Appendix~\ref{app:experiment}.

\paragraph{Overall} 
The main components of our hybrid structure are illustrated in Table~\ref{tab:hybrid_rssm}. The total loss is:
\begin{equation}
\begin{aligned}
    \mathcal{L}(\phi):=\operatorname{E}_{q_{\phi}}\Big{[}\textstyle\sum_{t=1}^{T}\big{(}&\mathcal{L}_{\mathrm{dyn}}(\phi)+\mathcal{L}_{\mathrm{rec}}(\phi)+\\&\mathcal{L}_{\mathrm{sim}}(\phi)+\mathcal{L}_{\mathrm{pred}}(\phi)\big{)}\Big{]},
\end{aligned} 
\end{equation}
All components are optimized concurrently, with the joint minimization of the loss function with respect to the parameter $\phi$, encompassing all model parameters, using the Adam optimizer~\cite{kingma2014adam}. Notably, the additional terms introduced do not require any extra user-specified hyperparameters, which is easy to optimize in practice.

\paragraph{Learning to control}
With the latent representation and dynamics model, we perform actor-critic policy learning by rolling out trajectories in the latent space. The critic $v_\psi(s_t)$ is trained to predict the discounted cumulative reward given a latent state, and the actor $\pi_\psi(s_t)$ is trained to take the action that maximizes the critic’s prediction, which follows actor-critic training in DreamerV3~\cite{hafner2023mastering}.

\begin{figure}[t]
\centerline{\includegraphics[width=0.7\linewidth]{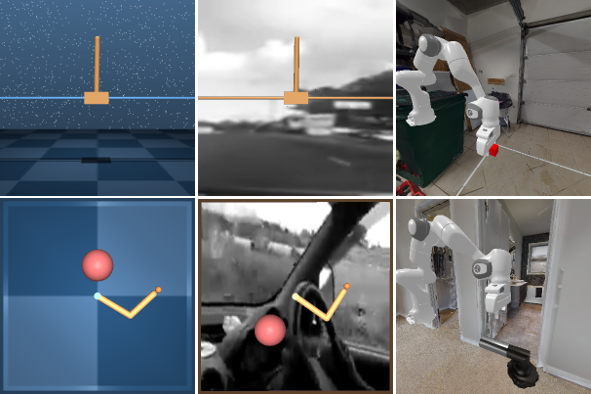}}
    \caption{Pixel observations of the DeepMind Control suite (left column) for \textit{cartpole} (top) and \textit{reacher} (bottom), Distracted DeepMind Control
suite (middle column) for \textit{cartpole} (top) and \textit{reacher} (bottom), and Mani-skill2 environments with distractions (right column) for \textit{cube} (top) and \textit{faucet} (bottom). \vspace{-5mm}   }
    \label{fig:envs}
\end{figure}

\section{Analysis}
\label{sec:analysis}
Our primary goal is to learn good state representations by focusing on two key objectives: latent reconstruction via a masking strategy for compact representations, and employing behavioral similarity for efficient representations. This section will highlight  both components are essential for our world model, underscoring their necessity.

Consider an MDP $\mathcal{M}$ as defined in Section~\ref{sec:prelim}, with vectorized state variables $\zeta=[s;\xi]$, where $\xi=[\xi^0 \xi^1 ... \xi^{n-1}]$ is a $n$-dim vector. We begin with an ideal assumption that our masking strategy only applies on exogenous noise $\xi$, \textit{i.e.}, $\tilde{\xi}\subseteq \xi$ be an arbitrary subset (a mask) and $\bar{\xi}=\xi\backslash\tilde{\xi}$ be the variables not included in the mask. Then the state reduces to $\bar{\zeta}=[s; \bar{\xi}]$. And we would like to know if the policy $\bar{\pi}$ under reduced MDP $\bar{\mathcal{M}}$ still being optimal for original MDP $\mathcal{M}$.
\begin{theorem}
If (1) $r(s_t,\xi^i_t,a_t)=0 \forall \xi^i \in \bar{\xi}$, (2) $P(s_{t+1}|s_{t},\xi,a_{t})=P(s_{t+1}|s_{t},\tilde{\xi},a_{t})$, and (3) $P(\tilde{\xi}_{t+1},\bar{\xi}_{t+1}|\tilde{\xi}_t,\bar{\xi}_t)=P(\tilde{\xi}_{t+1}|\tilde{\xi}_t)\cdot P(\bar{\xi}_{t+1}|\bar{\xi}_t)$, then we have $\bar{V}_{\bar{\pi}}(\bar{\zeta})=V_{\bar{\pi}}(\zeta) \forall \zeta \in \mathcal{Z}$, where $\bar{V}_{\bar{\pi}}(\bar{\zeta})$ is the value function under reduced MDP. If $\bar{\pi}$ is optimal for $\bar{\mathcal{M}}$, then $\bar{V}_{\bar{\pi}}(\bar{\zeta})=V^*(\zeta) \forall \zeta\in\mathcal{Z}$.
\end{theorem}
\begin{proof}
    See Appendix~\ref{app:proof}.
\end{proof}

It reveals that if we can identify and eliminate exogenous noise without altering the reward  or the internal dynamics of the underlying MDP, the resulting value function of this underlying MDP remains optimal with respect to the original problem. This scenario presents an opportunity for implementing a masking strategy. In practical settings, however, our masking approach involves random patch removal. This randomness does not guarantee the exclusive elimination of exogenous noise. Since elements of the environment crucial to the task may inadvertently be masked, the reward and dynamics can be incorrectly reconstructed, hence the underlying MDP (in latent space) is possibly changed. Consequently, if the masking technique is not sensitive to both the reward and the internal dynamics of the system, an optimal policy can not be assured. This limitation underscores why relying solely on masking-based latent reconstruction is insufficient for learning an effective world model in environments with distractions. Fortunately, the bisimulation principle offers a promising solution. By leveraging this principle, as detailed in Appendix~\ref{app:bisimulation}, we can train representations that encapsulate both reward and dynamic information. With bisimulation, the agent can be aware of the reward and the internal dynamics, and therefore can further update towards the optimal policies.


\begin{figure*}[t]
  \begin{center}
  \centerline{\includegraphics[width=0.7\textwidth]{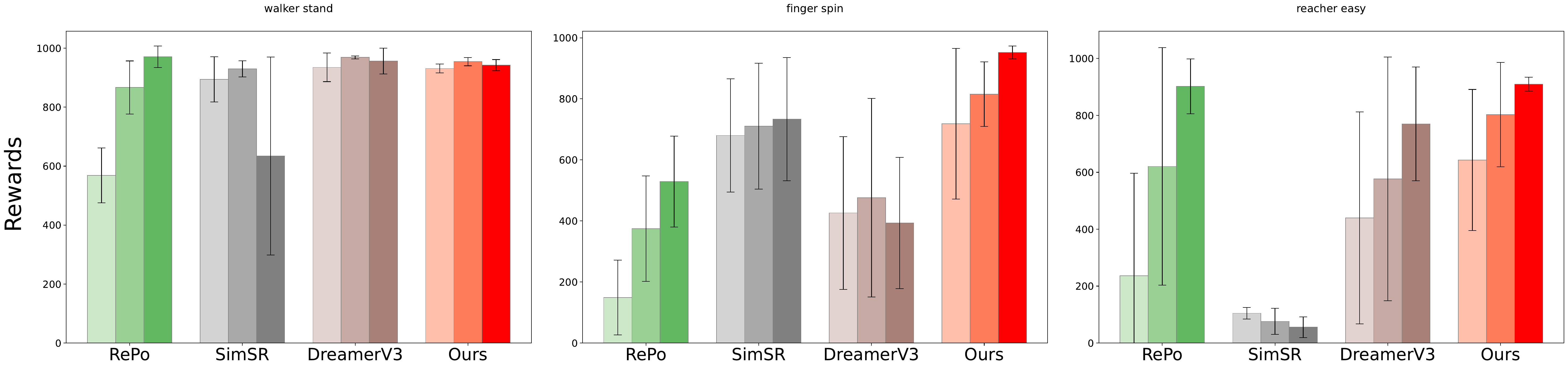}}
  \caption{Performance comparison on DMC tasks over 6 seeds in the default setting. Colors from light to dark represent the results evaluated at 100k, 250k, and 500k training steps, respectively, with different colors indicating different models.}
  \label{fig:defaut_results}
  \end{center}
  \end{figure*}

On the other hand, learning state reperesentation only with bisimulation objective is also not sufficient enough for model-based control. In model-based framework, integrating bisimulation-based objective requires to sample consecutive state pairs from an approximate dynamics model, \textit{e.g.}, RSSM in this paper. Though bisimulation objective has practically shown effectiveness in model-free settings~\cite{zhang2020invariant,zang2022simsr}, ~\cite{kemertas2021towards} illustrates that when refer to an approximate dynamics model, this dynamics model needs to meet certain condition to ensure the convergence of the bisimulation principle:
\begin{theorem}~\cite{kemertas2021towards}
    Assume $\mathcal{S}$ is compact. For $d^\pi$, if the support of an approximate dynamics model $\hat{P}$, $\operatorname{supp}(\hat{P})$ is a closed subset of $\mathcal{S}$, then there exists a unique fixed-point $d^\pi$, and this metric is bounded: $\operatorname{supp}(\hat{P}) \subseteq \mathcal{S} \Rightarrow \operatorname{diam}\left(\mathcal{S} ; {d}^\pi\right) \leq \frac{1}{1-\gamma}\left(R_{\max }-R_{\min }\right)$.
\end{theorem}
In practice, the support of an approximate dynamics model cannot be assured to be a subset of the observation space due to the presence of unpredictable exogenous noise. Consequently, when exogenous noise is involved, objectives dependent on transition dynamics, including bisimulation objectives, are likely incapable of filtering out all task-irrelevant information. A numerical counterexample illustrating this point is provided in Appendix~\ref{app:proof}. Therefore, to effectively reduce spatio-temporal redundancy in the observation space, additional methods are necessary. This is the rationale behind our adoption of a masking strategy and latent reconstruction in our approach.

\section{Experiments}
\label{sec:experiments}
\begin{figure*}[h]
  \begin{center}
  \includegraphics[width=0.55\linewidth]{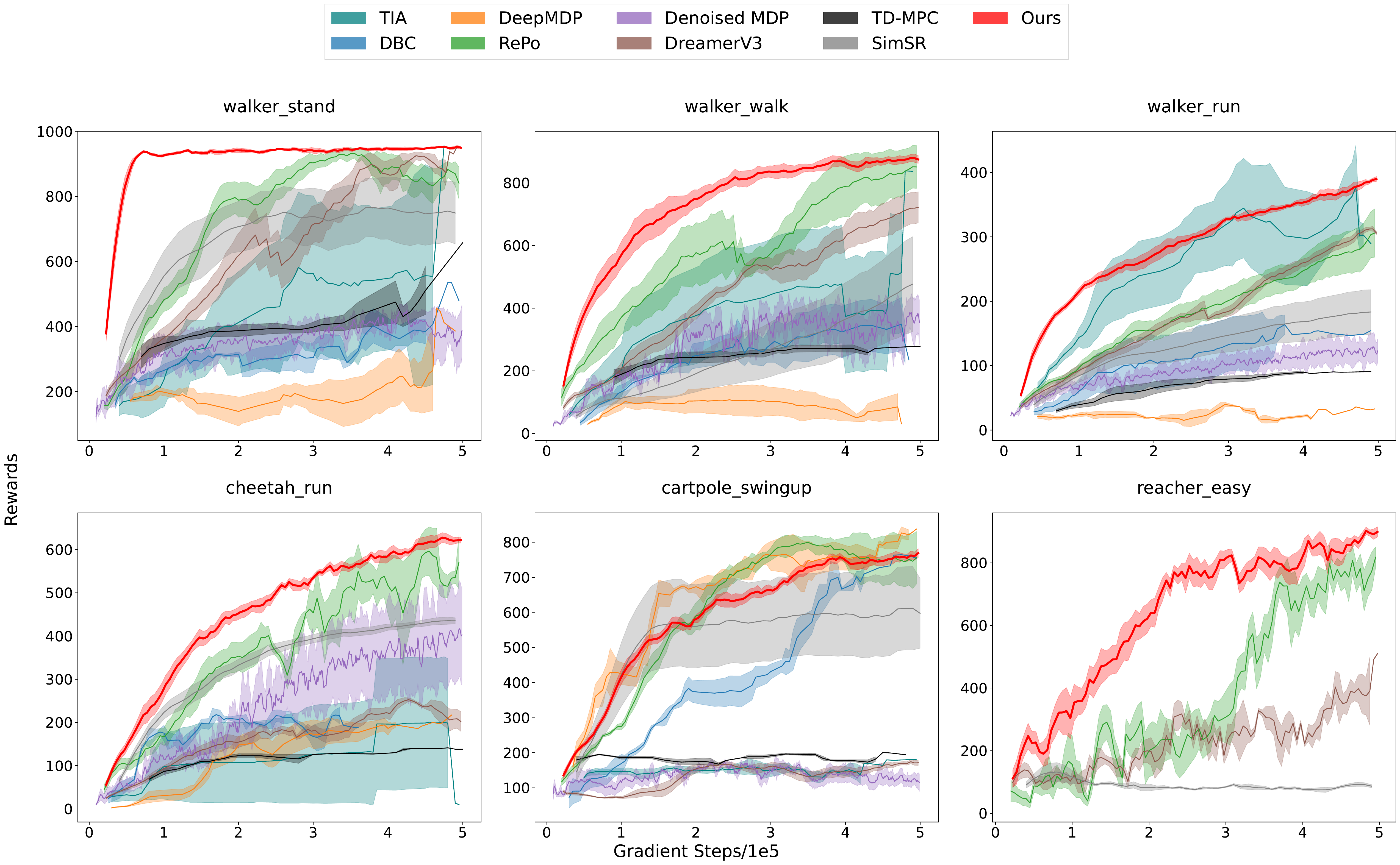}
  \caption{Performance comparison on DMC tasks with one standard error shaded in the distraction setting. The horizontal axis indicates the number of gradient steps. The vertical axis represents the mean return. For our model, RePo, along with DreamerV3 and SimSR, the returns are averaged with six random seeds. For the remaining models, the returns are averaged with three seeds.
  \vspace{-4mm}  
  }
  \label{fig:dmc_distraction}
  \end{center}
  \end{figure*}
We aim to address the following questions through our experiments : (1) Compared to prior approaches, does our decoder-free model weaken the resulting performance of the policy on downstream tasks? (2) Can we learn effective world models for planning, in the presence of environments containing exogenous  spatio-temporal noise structures? (3) We perform ablation studies showing the effectiveness of each of the components in our proposed model (4) Can the proposed Hybrid-RSSM architecture, along with the masking strategy, outperform state-of-the-art Dreamer based models, in presence of exogenous information in data? 

\begin{table*}[t]
  \centering
  \caption{Summary of performance metrics evaluated at 100K training steps. The performance is quantified in terms of the average score $\pm$ standard deviation. The highest result for each task is highlighted in bold. For RePo, DreamerV3, and our model, the returns are averaged with six random seeds. For TIA and Denoised MDP, the returns are averaged with three seeds.}
  \begin{tabular}{||l||lllll||l}
  \hline
  Task        & {DreamerV3}  & TIA     & Denoised MDP & RePo    & Ours   \\ \hline
  Lift Cube    & 167$\pm$45 & \textbf{274$\pm$173} & 155$\pm$71       & 83$\pm$22 & \textbf{284$\pm$85} \\
  Turn Faucet & 138$\pm$83   & 47$\pm$23   & 71$\pm$27        & 92$\pm$88   & \textbf{278$\pm$97}\\
  \hline
  \end{tabular}
  \vspace{-3mm}  
  \label{tab:maniskill results}
\end{table*}

\textbf{Experimental Setup} We evaluate our visual image-based continuous control tasks to assess their sample efficiency and overall performance. We perform our experiments in three distinct settings: 
i) a set of MuJoCo tasks~\cite{todorov2012mujoco} provided by  Deepmind Control(DMC) suite~\cite{tassa2018deepmind}, ii) a variant of DeepMind Control Suite where the background is replaced with grayscale natural videos from Kinetics dataset~\cite{kay2017kinetics}, termed as Distracted DeepMind Control Suite~\cite{zhang2018natural}, and iii) a benchmark based on the Maniskill2~\cite{gu2023maniskill2}, enhanced with realistic images of human homes~\cite{chang2017matterport3d} as backgrounds and was introduced in~\cite{DBLP:journals/corr/abs-2309-00082}.  Six tasks were tested in the first two settings and two in the last, with a total of 14 tasks. Task examples are depicted in Figure~\ref{fig:envs}.


\textbf{Baselines} We compare our proposed model against leading sample-efficient, model-free and model-based reinforcement learning (RL) methods in continuous control tasks. For model-free mehtods, our baselines include: 
DBC~\cite{zhang2020invariant} and SimSR~\cite{zang2022simsr}, both of which are two representative bisimulation-based methods. For model-based RL, experimental comparisons are made with TD-MPC~\cite{hansen2022temporal}, DreamerV3~\cite{hafner2023mastering}, and its extensions (TIA~\cite{fu2021learning}, Denoised MDP~\cite{wang2022denoised}, RePo~\cite{DBLP:journals/corr/abs-2309-00082}) that enhance robust representation learning. In this experiment, we use DreamerV3 as our backbone and build on top of it to develop our hybrid structure, where we use an unofficial open-sourced pytorch version of DreamerV3\cite{dreamerv3torch}. Notably, despite incorporating dual RSSMs, mask branch and raw branch in our framework, our model maintains a slightly smaller overall size compared to the original DreamerV3, which is notable considering the substantial size of the decoder parameters in DreamerV3. Detailed descriptions of our model are provided in the Appendix~\ref{app:model_detail}.


\textbf{Results on DMC tasks with default settings.} As shown in the left
column of Figure~\ref{fig:envs}, the default setting, which is provided by DMC, has simple backgrounds for the pixel observations.
Figure~\ref{fig:defaut_results} shows that our model consistently surpasses all baselines including RePo at 100k, 250k and 500k training steps in all three tasks, showcasing superior sample efficiency and final performance. Our model consistently equals or betters the performance of DreamerV3, illustrating our robustness against performance loss from omitting pixel-level reconstruction. This also highlights that the performance improvements of our model are primarily attributed to the innovative hybrid RSSM structure and objectives.

\textbf{Results on DMC tasks with distraction settings.} 
Figure~\ref{fig:dmc_distraction} illustrates our model's ability to ignore irrelevant information, outperforming most other models in various tasks. This underscores our method's resilience and efficiency in learning exo-free policies even in the presence of significant distractor information.Notably, we have almost the lowest variance across all tasks, which illustrates the robustness of our hybrid architecture, showing that our HRSSM is well-suited for model-based agents and is capable of learning compact and effective representations and dynamics.  However, in the \textit{cartpole\_swingup} task, our model slightly underperforms compared to DeepMDP and RePo. This may be due to our random masking strategy, which might inadvertently hide crucial elements like the small pole, crucial for task-relevant information. A learned masking strategy could be more effective than random masking in such cases, which is deserved to further investigation.

\textbf{Realistic Maniskill} 
Table~\ref{tab:maniskill results} not only demonstrates the competitive performance of our method but also underlines its distinct advantages in terms of consistency and robustness across different tasks. In the \textit{Lift Cube} task, our method achieved a competitive score of 284$\pm$85, paralleling TIA. However, the significantly lower variance in our results indicates superior consistency and reliability. This is critical in real-world scenarios where predictability and stability are as crucial as performance. In \textit{Turn Faucet}, our method's superiority is even more pronounced, substantially higher than its closest competitor. This not only showcases our method's ability to handle complex tasks efficiently but also its robust state representation.




\textbf{Ablation Studies} Our model comprises two key elements: mask-based latent reconstruction and a similarity objective guided by the bisimulation principle. We present their empirical impacts in the distraction setting of DMC tasks in Appendix~\ref{app:ablation}. To evaluate mask-based latent reconstruction, we eliminated the mask branch and reverted our hybrid RSSM to a standard RSSM, also omitting the cube masking and the latent reconstruction loss. For the bisimulation principle ablation, we simply removed the similarity loss. Results indicate that models lacking these components underperform relative to the full model, showcasing their critical importance in our framework.

\textbf{Training Time Comparison} 
As our hybrid structure incorporates two RSSMs, one might wonder about the computational efficiency of our framework. Notably, gradients are only backpropagated through the mask branch, while the parameters of the RSSM in the raw branch are updated via Exponential Moving Average (EMA). Moreover, since we utilize the same historical representation, the computational time required for the forward process is considerably less than twice as much. To validate this, we compared the wall-clock training time of our method against DreamerV3, with the results provided in Appendix~\ref{app:wall_clock}. These results confirm that our method is comparable to the original DreamerV3 in terms of computational efficiency, without incurring substantial additional time costs.



\section{Discussion}

\textbf{Limitations and Future Work} Our approach's potential limitation lies in the lack of a task-specific masking strategy, which could partially damage the endogenous state and slightly reduce the final performance. Future improvements could involve signal-to-noise ratios~\cite{tomar2023ignorance} to reduce the original image, aiming to identify the minimal information essential for the task.

\textbf{Conclusion: } In this paper, we presented a new framework to learn state representations and dynamics in the presence of exogenous noise. We introduced the masking strategy and latent reconstruction to eliminate redundant spatio-temporal information, and employed bisimulation principle to capture task-relevant information. Addressing co-training instabilities, we further developed a hybrid RSSM structure. Empirical results demonstrated the effectiveness of our model. 

\newpage

\section*{Impact Statement}
This paper synthesizes theoretical and empirical results to build more capable Model-Based Reinforcement Learning (MBRL) agents in settings with exogenous noise. In real-world applications, distractions are prevalent across different scenarios. Enabling model-based agents to learn control from such scenarios can be beneficial not only for solving complex tasks but also for increasing sample efficiency during deployments in environments with varying contexts. Our framework is not only theoretically sound but also technically straightforward to implement and empirically competitive. We believe that developing MBRL agents by focusing on the compactness and effectiveness of the representation and dynamics is an important step towards creating more applicable Artificial General Intelligence (AGI).

\section*{Acknowledgments}
The authors would like to thank Chuning Zhu for generously sharing the performance data of several baselines. The authors would also like to thank the anonymous reviewers for their constructive feedback, which greatly helped us to improve the quality and clarity of the paper. RI would like to thank the DreamFold team for providing workspace to carry out this work. This work was partially supported by the NSFC under Grants 92270125 and 62276024, as well as the National Key R\&D Program of China under Grant No.2022YFC3302101.

\bibliography{example_paper}

\begin{thebibliography}{77}
\providecommand{\natexlab}[1]{#1}
\providecommand{\url}[1]{\texttt{#1}}
\expandafter\ifx\csname urlstyle\endcsname\relax
  \providecommand{\doi}[1]{doi: #1}\else
  \providecommand{\doi}{doi: \begingroup \urlstyle{rm}\Url}\fi

\bibitem[Agarwal et~al.(2020)Agarwal, Kakade, Krishnamurthy, and Sun]{agarwal2020flambe}
Agarwal, A., Kakade, S., Krishnamurthy, A., and Sun, W.
\newblock Flambe: Structural complexity and representation learning of low rank mdps.
\newblock \emph{Advances in neural information processing systems}, 33:\penalty0 20095--20107, 2020.

\bibitem[Agarwal et~al.(2021)Agarwal, Schwarzer, Castro, Courville, and Bellemare]{DBLP:conf/nips/AgarwalSCCB21}
Agarwal, R., Schwarzer, M., Castro, P.~S., Courville, A.~C., and Bellemare, M.~G.
\newblock Deep reinforcement learning at the edge of the statistical precipice.
\newblock In Ranzato, M., Beygelzimer, A., Dauphin, Y.~N., Liang, P., and Vaughan, J.~W. (eds.), \emph{Advances in Neural Information Processing Systems 34: Annual Conference on Neural Information Processing Systems 2021, NeurIPS 2021, December 6-14, 2021, virtual}, pp.\  29304--29320, 2021.

\bibitem[Asadi et~al.(2019)Asadi, Misra, Kim, and Littman]{asadi2019combating}
Asadi, K., Misra, D., Kim, S., and Littman, M.~L.
\newblock Combating the compounding-error problem with a multi-step model.
\newblock \emph{arXiv preprint arXiv:1905.13320}, 2019.

\bibitem[Burda et~al.(2018)Burda, Edwards, Storkey, and Klimov]{burda2018exploration}
Burda, Y., Edwards, H., Storkey, A., and Klimov, O.
\newblock Exploration by random network distillation.
\newblock \emph{arXiv preprint arXiv:1810.12894}, 2018.

\bibitem[Caron et~al.(2020)Caron, Misra, Mairal, Goyal, Bojanowski, and Joulin]{caron2020unsupervised}
Caron, M., Misra, I., Mairal, J., Goyal, P., Bojanowski, P., and Joulin, A.
\newblock Unsupervised learning of visual features by contrasting cluster assignments.
\newblock \emph{Advances in neural information processing systems}, 33:\penalty0 9912--9924, 2020.

\bibitem[Castro(2020)]{castro2020scalable}
Castro, P.~S.
\newblock Scalable methods for computing state similarity in deterministic markov decision processes.
\newblock In \emph{Proceedings of the AAAI Conference on Artificial Intelligence}, volume~34, pp.\  10069--10076, 2020.

\bibitem[Castro et~al.(2021)Castro, Kastner, Panangaden, and Rowland]{castro2021mico}
Castro, P.~S., Kastner, T., Panangaden, P., and Rowland, M.
\newblock Mico: Improved representations via sampling-based state similarity for markov decision processes.
\newblock \emph{Advances in Neural Information Processing Systems}, 34:\penalty0 30113--30126, 2021.

\bibitem[Chang et~al.(2017)Chang, Dai, Funkhouser, Halber, Niessner, Savva, Song, Zeng, and Zhang]{chang2017matterport3d}
Chang, A., Dai, A., Funkhouser, T., Halber, M., Niessner, M., Savva, M., Song, S., Zeng, A., and Zhang, Y.
\newblock Matterport3d: Learning from rgb-d data in indoor environments.
\newblock \emph{arXiv preprint arXiv:1709.06158}, 2017.

\bibitem[Chitnis \& Lozano-P{\'e}rez(2020)Chitnis and Lozano-P{\'e}rez]{chitnis2020learning}
Chitnis, R. and Lozano-P{\'e}rez, T.
\newblock Learning compact models for planning with exogenous processes.
\newblock In \emph{Conference on Robot Learning}, pp.\  813--822. PMLR, 2020.

\bibitem[Cuturi(2013)]{cuturi2013sinkhorn}
Cuturi, M.
\newblock Sinkhorn distances: Lightspeed computation of optimal transport.
\newblock \emph{Advances in neural information processing systems}, 26, 2013.

\bibitem[Deng et~al.(2022)Deng, Jang, and Ahn]{deng2022dreamerpro}
Deng, F., Jang, I., and Ahn, S.
\newblock Dreamerpro: Reconstruction-free model-based reinforcement learning with prototypical representations.
\newblock In \emph{International Conference on Machine Learning}, pp.\  4956--4975. PMLR, 2022.

\bibitem[Efroni et~al.(2021)Efroni, Misra, Krishnamurthy, Agarwal, and Langford]{efroni2021provably}
Efroni, Y., Misra, D., Krishnamurthy, A., Agarwal, A., and Langford, J.
\newblock Provably filtering exogenous distractors using multistep inverse dynamics.
\newblock In \emph{International Conference on Learning Representations}, 2021.

\bibitem[Efroni et~al.(2022)Efroni, Foster, Misra, Krishnamurthy, and Langford]{efroni2022sample}
Efroni, Y., Foster, D.~J., Misra, D., Krishnamurthy, A., and Langford, J.
\newblock Sample-efficient reinforcement learning in the presence of exogenous information.
\newblock In \emph{Conference on Learning Theory}, pp.\  5062--5127. PMLR, 2022.

\bibitem[Feichtenhofer et~al.(2022)Feichtenhofer, Fan, Li, and He]{videomae}
Feichtenhofer, C., Fan, H., Li, Y., and He, K.
\newblock Masked autoencoders as spatiotemporal learners.
\newblock \emph{CoRR}, abs/2205.09113, 2022.
\newblock \doi{10.48550/arXiv.2205.09113}.
\newblock URL \url{https://doi.org/10.48550/arXiv.2205.09113}.

\bibitem[Ferns et~al.(2011)Ferns, Panangaden, and Precup]{ferns2011bisimulation}
Ferns, N., Panangaden, P., and Precup, D.
\newblock Bisimulation metrics for continuous markov decision processes.
\newblock \emph{SIAM Journal on Computing}, 40\penalty0 (6):\penalty0 1662--1714, 2011.

\bibitem[Ferns et~al.(2012{\natexlab{a}})Ferns, Castro, Precup, and Panangaden]{ferns2012methods}
Ferns, N., Castro, P.~S., Precup, D., and Panangaden, P.
\newblock Methods for computing state similarity in markov decision processes.
\newblock \emph{arXiv preprint arXiv:1206.6836}, 2012{\natexlab{a}}.

\bibitem[Ferns et~al.(2012{\natexlab{b}})Ferns, Panangaden, and Precup]{ferns2012metrics}
Ferns, N., Panangaden, P., and Precup, D.
\newblock Metrics for finite markov decision processes.
\newblock \emph{arXiv preprint arXiv:1207.4114}, 2012{\natexlab{b}}.

\bibitem[Fu et~al.(2021)Fu, Yang, Agrawal, and Jaakkola]{fu2021learning}
Fu, X., Yang, G., Agrawal, P., and Jaakkola, T.
\newblock Learning task informed abstractions.
\newblock In \emph{International Conference on Machine Learning}, pp.\  3480--3491. PMLR, 2021.

\bibitem[Gelada et~al.(2019)Gelada, Kumar, Buckman, Nachum, and Bellemare]{gelada2019deepmdp}
Gelada, C., Kumar, S., Buckman, J., Nachum, O., and Bellemare, M.~G.
\newblock Deepmdp: Learning continuous latent space models for representation learning.
\newblock In \emph{International Conference on Machine Learning}, pp.\  2170--2179. PMLR, 2019.

\bibitem[Givan et~al.(2003)Givan, Dean, and Greig]{DBLP:journals/ai/GivanDG03}
Givan, R., Dean, T.~L., and Greig, M.
\newblock Equivalence notions and model minimization in markov decision processes.
\newblock \emph{Artif. Intell.}, 147\penalty0 (1-2):\penalty0 163--223, 2003.

\bibitem[Gu et~al.(2023)Gu, Xiang, Li, Ling, Liu, Mu, Tang, Tao, Wei, Yao, et~al.]{gu2023maniskill2}
Gu, J., Xiang, F., Li, X., Ling, Z., Liu, X., Mu, T., Tang, Y., Tao, S., Wei, X., Yao, Y., et~al.
\newblock Maniskill2: A unified benchmark for generalizable manipulation skills.
\newblock \emph{arXiv preprint arXiv:2302.04659}, 2023.

\bibitem[Ha \& Schmidhuber(2018)Ha and Schmidhuber]{ha2018recurrent}
Ha, D. and Schmidhuber, J.
\newblock Recurrent world models facilitate policy evolution.
\newblock \emph{Advances in neural information processing systems}, 31, 2018.

\bibitem[Hafner et~al.(2019{\natexlab{a}})Hafner, Lillicrap, Ba, and Norouzi]{hafner2019dream}
Hafner, D., Lillicrap, T., Ba, J., and Norouzi, M.
\newblock Dream to control: Learning behaviors by latent imagination.
\newblock \emph{arXiv preprint arXiv:1912.01603}, 2019{\natexlab{a}}.

\bibitem[Hafner et~al.(2019{\natexlab{b}})Hafner, Lillicrap, Fischer, Villegas, Ha, Lee, and Davidson]{hafner2019learning}
Hafner, D., Lillicrap, T., Fischer, I., Villegas, R., Ha, D., Lee, H., and Davidson, J.
\newblock Learning latent dynamics for planning from pixels.
\newblock In \emph{International conference on machine learning}, pp.\  2555--2565. PMLR, 2019{\natexlab{b}}.

\bibitem[Hafner et~al.(2020)Hafner, Lillicrap, Norouzi, and Ba]{hafner2020mastering}
Hafner, D., Lillicrap, T., Norouzi, M., and Ba, J.
\newblock Mastering atari with discrete world models.
\newblock \emph{arXiv preprint arXiv:2010.02193}, 2020.

\bibitem[Hafner et~al.(2023)Hafner, Pasukonis, Ba, and Lillicrap]{hafner2023mastering}
Hafner, D., Pasukonis, J., Ba, J., and Lillicrap, T.
\newblock Mastering diverse domains through world models.
\newblock \emph{arXiv preprint arXiv:2301.04104}, 2023.

\bibitem[Hamrick(2019)]{hamrick2019analogues}
Hamrick, J.~B.
\newblock Analogues of mental simulation and imagination in deep learning.
\newblock \emph{Current Opinion in Behavioral Sciences}, 29:\penalty0 8--16, 2019.

\bibitem[Hansen \& Wang(2021)Hansen and Wang]{hansen2021generalization}
Hansen, N. and Wang, X.
\newblock Generalization in reinforcement learning by soft data augmentation.
\newblock In \emph{2021 IEEE International Conference on Robotics and Automation (ICRA)}, pp.\  13611--13617. IEEE, 2021.

\bibitem[Hansen et~al.(2021)Hansen, Su, and Wang]{hansen2021stabilizing}
Hansen, N., Su, H., and Wang, X.
\newblock Stabilizing deep q-learning with convnets and vision transformers under data augmentation.
\newblock \emph{Advances in neural information processing systems}, 34:\penalty0 3680--3693, 2021.

\bibitem[Hansen et~al.(2022{\natexlab{a}})Hansen, Lin, Su, Wang, Kumar, and Rajeswaran]{hansen2022modem}
Hansen, N., Lin, Y., Su, H., Wang, X., Kumar, V., and Rajeswaran, A.
\newblock Modem: Accelerating visual model-based reinforcement learning with demonstrations.
\newblock \emph{arXiv preprint arXiv:2212.05698}, 2022{\natexlab{a}}.

\bibitem[Hansen et~al.(2022{\natexlab{b}})Hansen, Wang, and Su]{hansen2022temporal}
Hansen, N., Wang, X., and Su, H.
\newblock Temporal difference learning for model predictive control.
\newblock \emph{arXiv preprint arXiv:2203.04955}, 2022{\natexlab{b}}.

\bibitem[Hansen et~al.(2023)Hansen, Su, and Wang]{hansen2023td}
Hansen, N., Su, H., and Wang, X.
\newblock Td-mpc2: Scalable, robust world models for continuous control.
\newblock \emph{arXiv preprint arXiv:2310.16828}, 2023.

\bibitem[He et~al.(2022)He, Chen, Xie, Li, Doll{\'a}r, and Girshick]{he2022masked}
He, K., Chen, X., Xie, S., Li, Y., Doll{\'a}r, P., and Girshick, R.
\newblock Masked autoencoders are scalable vision learners.
\newblock In \emph{Proceedings of the IEEE/CVF conference on computer vision and pattern recognition}, pp.\  16000--16009, 2022.

\bibitem[Henderson et~al.(2018)Henderson, Islam, Bachman, Pineau, Precup, and Meger]{islamrepro}
Henderson, P., Islam, R., Bachman, P., Pineau, J., Precup, D., and Meger, D.
\newblock Deep reinforcement learning that matters.
\newblock In McIlraith, S.~A. and Weinberger, K.~Q. (eds.), \emph{Proceedings of the Thirty-Second {AAAI} Conference on Artificial Intelligence, (AAAI-18), the 30th innovative Applications of Artificial Intelligence (IAAI-18), and the 8th {AAAI} Symposium on Educational Advances in Artificial Intelligence (EAAI-18), New Orleans, Louisiana, USA, February 2-7, 2018}, pp.\  3207--3214. {AAAI} Press, 2018.
\newblock \doi{10.1609/AAAI.V32I1.11694}.
\newblock URL \url{https://doi.org/10.1609/aaai.v32i1.11694}.

\bibitem[Islam et~al.(2022)Islam, Tomar, Lamb, Efroni, Zang, Didolkar, Misra, Li, van Seijen, Combes, et~al.]{islam2022agent}
Islam, R., Tomar, M., Lamb, A., Efroni, Y., Zang, H., Didolkar, A., Misra, D., Li, X., van Seijen, H., Combes, R. T.~d., et~al.
\newblock Agent-controller representations: Principled offline rl with rich exogenous information.
\newblock \emph{arXiv preprint arXiv:2211.00164}, 2022.

\bibitem[Islam et~al.(2023)Islam, Tomar, Lamb, Efroni, Zang, Didolkar, Misra, Li, Van~Seijen, Des~Combes, et~al.]{islam2023principled}
Islam, R., Tomar, M., Lamb, A., Efroni, Y., Zang, H., Didolkar, A.~R., Misra, D., Li, X., Van~Seijen, H., Des~Combes, R.~T., et~al.
\newblock Principled offline rl in the presence of rich exogenous information.
\newblock 2023.

\bibitem[Janner et~al.(2019)Janner, Fu, Zhang, and Levine]{janner2019trust}
Janner, M., Fu, J., Zhang, M., and Levine, S.
\newblock When to trust your model: Model-based policy optimization.
\newblock \emph{Advances in neural information processing systems}, 32, 2019.

\bibitem[Janner et~al.(2020)Janner, Mordatch, and Levine]{janner2020gamma}
Janner, M., Mordatch, I., and Levine, S.
\newblock gamma-models: Generative temporal difference learning for infinite-horizon prediction.
\newblock \emph{Advances in Neural Information Processing Systems}, 33:\penalty0 1724--1735, 2020.

\bibitem[Kalweit \& Boedecker(2017)Kalweit and Boedecker]{kalweit2017uncertainty}
Kalweit, G. and Boedecker, J.
\newblock Uncertainty-driven imagination for continuous deep reinforcement learning.
\newblock In \emph{Conference on Robot Learning}, pp.\  195--206. PMLR, 2017.

\bibitem[Kay et~al.(2017)Kay, Carreira, Simonyan, Zhang, Hillier, Vijayanarasimhan, Viola, Green, Back, Natsev, et~al.]{kay2017kinetics}
Kay, W., Carreira, J., Simonyan, K., Zhang, B., Hillier, C., Vijayanarasimhan, S., Viola, F., Green, T., Back, T., Natsev, P., et~al.
\newblock The kinetics human action video dataset.
\newblock \emph{arXiv preprint arXiv:1705.06950}, 2017.

\bibitem[Kemertas \& Aumentado-Armstrong(2021)Kemertas and Aumentado-Armstrong]{kemertas2021towards}
Kemertas, M. and Aumentado-Armstrong, T.
\newblock Towards robust bisimulation metric learning.
\newblock \emph{Advances in Neural Information Processing Systems}, 34:\penalty0 4764--4777, 2021.

\bibitem[Kingma \& Ba(2014)Kingma and Ba]{kingma2014adam}
Kingma, D.~P. and Ba, J.
\newblock Adam: A method for stochastic optimization.
\newblock \emph{arXiv preprint arXiv:1412.6980}, 2014.

\bibitem[Kingma et~al.(2016)Kingma, Salimans, Jozefowicz, Chen, Sutskever, and Welling]{kingma2016improved}
Kingma, D.~P., Salimans, T., Jozefowicz, R., Chen, X., Sutskever, I., and Welling, M.
\newblock Improved variational inference with inverse autoregressive flow.
\newblock \emph{Advances in neural information processing systems}, 29, 2016.

\bibitem[Kostrikov et~al.(2020)Kostrikov, Yarats, and Fergus]{kostrikov2020image}
Kostrikov, I., Yarats, D., and Fergus, R.
\newblock Image augmentation is all you need: Regularizing deep reinforcement learning from pixels.
\newblock \emph{arXiv preprint arXiv:2004.13649}, 2020.

\bibitem[Lamb et~al.(2022)Lamb, Islam, Efroni, Didolkar, Misra, Foster, Molu, Chari, Krishnamurthy, and Langford]{lamb2022guaranteed}
Lamb, A., Islam, R., Efroni, Y., Didolkar, A.~R., Misra, D., Foster, D.~J., Molu, L.~P., Chari, R., Krishnamurthy, A., and Langford, J.
\newblock Guaranteed discovery of control-endogenous latent states with multi-step inverse models.
\newblock \emph{Transactions on Machine Learning Research}, 2022.

\bibitem[Larsen \& Skou(1989)Larsen and Skou]{DBLP:conf/popl/LarsenS89}
Larsen, K.~G. and Skou, A.
\newblock Bisimulation through probabilistic testing.
\newblock In \emph{Conference Record of the Sixteenth Annual {ACM} Symposium on Principles of Programming Languages, Austin, Texas, USA, January 11-13, 1989}, pp.\  344--352. {ACM} Press, 1989.

\bibitem[Laskin et~al.(2020)Laskin, Srinivas, and Abbeel]{laskin2020curl}
Laskin, M., Srinivas, A., and Abbeel, P.
\newblock Curl: Contrastive unsupervised representations for reinforcement learning.
\newblock In \emph{International Conference on Machine Learning}, pp.\  5639--5650. PMLR, 2020.

\bibitem[Liu et~al.(2022)Liu, Liu, Grover, and Abbeel]{liu2022masked}
Liu, F., Liu, H., Grover, A., and Abbeel, P.
\newblock Masked autoencoding for scalable and generalizable decision making.
\newblock \emph{Advances in Neural Information Processing Systems}, 35:\penalty0 12608--12618, 2022.

\bibitem[{NM512}(2023)]{dreamerv3torch}
{NM512}.
\newblock Dreamerv3 pytorch implementation.
\newblock \url{https://github.com/NM512/dreamerv3-torch}, 2023.

\bibitem[Okada \& Taniguchi(2021)Okada and Taniguchi]{okada2021dreaming}
Okada, M. and Taniguchi, T.
\newblock Dreaming: Model-based reinforcement learning by latent imagination without reconstruction.
\newblock In \emph{2021 ieee international conference on robotics and automation (icra)}, pp.\  4209--4215. IEEE, 2021.

\bibitem[Poudel et~al.(2023)Poudel, Pandya, Liwicki, and Cipolla]{poudel2023recore}
Poudel, R.~P., Pandya, H., Liwicki, S., and Cipolla, R.
\newblock Recore: Regularized contrastive representation learning of world model.
\newblock \emph{arXiv preprint arXiv:2312.09056}, 2023.

\bibitem[Rafailov et~al.(2021)Rafailov, Yu, Rajeswaran, and Finn]{rafailov2021offline}
Rafailov, R., Yu, T., Rajeswaran, A., and Finn, C.
\newblock Offline reinforcement learning from images with latent space models.
\newblock In \emph{Learning for Dynamics and Control}, pp.\  1154--1168. PMLR, 2021.

\bibitem[Ramakrishnan et~al.(2021)Ramakrishnan, Gokaslan, Wijmans, Maksymets, Clegg, Turner, Undersander, Galuba, Westbury, Chang, et~al.]{ramakrishnan2021habitat}
Ramakrishnan, S.~K., Gokaslan, A., Wijmans, E., Maksymets, O., Clegg, A., Turner, J., Undersander, E., Galuba, W., Westbury, A., Chang, A.~X., et~al.
\newblock Habitat-matterport 3d dataset (hm3d): 1000 large-scale 3d environments for embodied ai.
\newblock \emph{arXiv preprint arXiv:2109.08238}, 2021.

\bibitem[Selvaraju et~al.(2017)Selvaraju, Cogswell, Das, Vedantam, Parikh, and Batra]{selvaraju2017grad}
Selvaraju, R.~R., Cogswell, M., Das, A., Vedantam, R., Parikh, D., and Batra, D.
\newblock Grad-cam: Visual explanations from deep networks via gradient-based localization.
\newblock In \emph{Proceedings of the IEEE international conference on computer vision}, pp.\  618--626, 2017.

\bibitem[Seo et~al.(2023{\natexlab{a}})Seo, Hafner, Liu, Liu, James, Lee, and Abbeel]{seo2023masked}
Seo, Y., Hafner, D., Liu, H., Liu, F., James, S., Lee, K., and Abbeel, P.
\newblock Masked world models for visual control.
\newblock In \emph{Conference on Robot Learning}, pp.\  1332--1344. PMLR, 2023{\natexlab{a}}.

\bibitem[Seo et~al.(2023{\natexlab{b}})Seo, Kim, James, Lee, Shin, and Abbeel]{seo2023multi}
Seo, Y., Kim, J., James, S., Lee, K., Shin, J., and Abbeel, P.
\newblock Multi-view masked world models for visual robotic manipulation.
\newblock \emph{arXiv preprint arXiv:2302.02408}, 2023{\natexlab{b}}.

\bibitem[Stone et~al.(2021)Stone, Ramirez, Konolige, and Jonschkowski]{DBLP:journals/corr/abs-2101-02722}
Stone, A., Ramirez, O., Konolige, K., and Jonschkowski, R.
\newblock The distracting control suite - {A} challenging benchmark for reinforcement learning from pixels.
\newblock \emph{CoRR}, abs/2101.02722, 2021.
\newblock URL \url{https://arxiv.org/abs/2101.02722}.

\bibitem[Sutton(1990)]{sutton1990integrated}
Sutton, R.~S.
\newblock Integrated architectures for learning, planning, and reacting based on approximating dynamic programming.
\newblock In \emph{Machine learning proceedings 1990}, pp.\  216--224. Elsevier, 1990.

\bibitem[Tassa et~al.(2018)Tassa, Doron, Muldal, Erez, Li, Casas, Budden, Abdolmaleki, Merel, Lefrancq, et~al.]{tassa2018deepmind}
Tassa, Y., Doron, Y., Muldal, A., Erez, T., Li, Y., Casas, D. d.~L., Budden, D., Abdolmaleki, A., Merel, J., Lefrancq, A., et~al.
\newblock Deepmind control suite.
\newblock \emph{arXiv preprint arXiv:1801.00690}, 2018.

\bibitem[Todorov et~al.(2012)Todorov, Erez, and Tassa]{todorov2012mujoco}
Todorov, E., Erez, T., and Tassa, Y.
\newblock Mujoco: A physics engine for model-based control.
\newblock In \emph{2012 IEEE/RSJ international conference on intelligent robots and systems}, pp.\  5026--5033. IEEE, 2012.

\bibitem[Tomar et~al.(2023)Tomar, Islam, Levine, and Bachman]{tomar2023ignorance}
Tomar, M., Islam, R., Levine, S., and Bachman, P.
\newblock Ignorance is bliss: Robust control via information gating.
\newblock \emph{arXiv preprint arXiv:2303.06121}, 2023.

\bibitem[Tong et~al.(2022)Tong, Song, Wang, and Wang]{tong2022videomae}
Tong, Z., Song, Y., Wang, J., and Wang, L.
\newblock Videomae: Masked autoencoders are data-efficient learners for self-supervised video pre-training.
\newblock \emph{Advances in neural information processing systems}, 35:\penalty0 10078--10093, 2022.

\bibitem[Wang et~al.(2022{\natexlab{a}})Wang, Du, Torralba, Isola, Zhang, and Tian]{wang2022denoised}
Wang, T., Du, S.~S., Torralba, A., Isola, P., Zhang, A., and Tian, Y.
\newblock Denoised mdps: Learning world models better than the world itself.
\newblock \emph{arXiv preprint arXiv:2206.15477}, 2022{\natexlab{a}}.

\bibitem[Wang et~al.(2022{\natexlab{b}})Wang, Xiao, Xu, Zhu, and Stone]{wang2022causal}
Wang, Z., Xiao, X., Xu, Z., Zhu, Y., and Stone, P.
\newblock Causal dynamics learning for task-independent state abstraction.
\newblock \emph{arXiv preprint arXiv:2206.13452}, 2022{\natexlab{b}}.

\bibitem[Wei et~al.(2022)Wei, Fan, Xie, Wu, Yuille, and Feichtenhofer]{wei2022masked}
Wei, C., Fan, H., Xie, S., Wu, C.-Y., Yuille, A., and Feichtenhofer, C.
\newblock Masked feature prediction for self-supervised visual pre-training.
\newblock In \emph{Proceedings of the IEEE/CVF Conference on Computer Vision and Pattern Recognition}, pp.\  14668--14678, 2022.

\bibitem[Williams et~al.(2017)Williams, Aldrich, and Theodorou]{williams2017model}
Williams, G., Aldrich, A., and Theodorou, E.~A.
\newblock Model predictive path integral control: From theory to parallel computation.
\newblock \emph{Journal of Guidance, Control, and Dynamics}, 40\penalty0 (2):\penalty0 344--357, 2017.

\bibitem[Xiao et~al.(2019)Xiao, Wu, Ma, Schuurmans, and M{\"u}ller]{xiao2019learning}
Xiao, C., Wu, Y., Ma, C., Schuurmans, D., and M{\"u}ller, M.
\newblock Learning to combat compounding-error in model-based reinforcement learning.
\newblock \emph{arXiv preprint arXiv:1912.11206}, 2019.

\bibitem[Yarats et~al.(2021{\natexlab{a}})Yarats, Fergus, Lazaric, and Pinto]{yarats2021mastering}
Yarats, D., Fergus, R., Lazaric, A., and Pinto, L.
\newblock Mastering visual continuous control: Improved data-augmented reinforcement learning.
\newblock \emph{arXiv preprint arXiv:2107.09645}, 2021{\natexlab{a}}.

\bibitem[Yarats et~al.(2021{\natexlab{b}})Yarats, Zhang, Kostrikov, Amos, Pineau, and Fergus]{yarats2021improving}
Yarats, D., Zhang, A., Kostrikov, I., Amos, B., Pineau, J., and Fergus, R.
\newblock Improving sample efficiency in model-free reinforcement learning from images.
\newblock In \emph{Proceedings of the AAAI Conference on Artificial Intelligence}, volume~35, pp.\  10674--10681, 2021{\natexlab{b}}.

\bibitem[Yu et~al.(2022)Yu, Zhang, Lan, Lu, and Chen]{yu2022mask}
Yu, T., Zhang, Z., Lan, C., Lu, Y., and Chen, Z.
\newblock Mask-based latent reconstruction for reinforcement learning.
\newblock \emph{Advances in Neural Information Processing Systems}, 35:\penalty0 25117--25131, 2022.

\bibitem[Zang et~al.(2022{\natexlab{a}})Zang, Li, and Wang]{zang2022simsr}
Zang, H., Li, X., and Wang, M.
\newblock Simsr: Simple distance-based state representations for deep reinforcement learning.
\newblock In \emph{Proceedings of the AAAI Conference on Artificial Intelligence}, volume~36, pp.\  8997--9005, 2022{\natexlab{a}}.

\bibitem[Zang et~al.(2022{\natexlab{b}})Zang, Li, Yu, Liu, Islam, Combes, and Laroche]{zang2022behavior}
Zang, H., Li, X., Yu, J., Liu, C., Islam, R., Combes, R. T.~D., and Laroche, R.
\newblock Behavior prior representation learning for offline reinforcement learning.
\newblock \emph{arXiv preprint arXiv:2211.00863}, 2022{\natexlab{b}}.

\bibitem[Zhang et~al.(2018)Zhang, Wu, and Pineau]{zhang2018natural}
Zhang, A., Wu, Y., and Pineau, J.
\newblock Natural environment benchmarks for reinforcement learning.
\newblock \emph{arXiv preprint arXiv:1811.06032}, 2018.

\bibitem[Zhang et~al.(2020{\natexlab{a}})Zhang, Lyle, Sodhani, Filos, Kwiatkowska, Pineau, Gal, and Precup]{zhang2020invariant}
Zhang, A., Lyle, C., Sodhani, S., Filos, A., Kwiatkowska, M., Pineau, J., Gal, Y., and Precup, D.
\newblock Invariant causal prediction for block mdps.
\newblock In \emph{International Conference on Machine Learning}, pp.\  11214--11224. PMLR, 2020{\natexlab{a}}.

\bibitem[Zhang et~al.(2020{\natexlab{b}})Zhang, McAllister, Calandra, Gal, and Levine]{zhang2020learning}
Zhang, A., McAllister, R., Calandra, R., Gal, Y., and Levine, S.
\newblock Learning invariant representations for reinforcement learning without reconstruction.
\newblock \emph{arXiv preprint arXiv:2006.10742}, 2020{\natexlab{b}}.

\bibitem[Zhao et~al.(2023)Zhao, Zhao, Boney, Kannala, and Pajarinen]{zhao2023simplified}
Zhao, Y., Zhao, W., Boney, R., Kannala, J., and Pajarinen, J.
\newblock Simplified temporal consistency reinforcement learning.
\newblock \emph{arXiv preprint arXiv:2306.09466}, 2023.

\bibitem[Zhu et~al.(2023)Zhu, Simchowitz, Gadipudi, and Gupta]{DBLP:journals/corr/abs-2309-00082}
Zhu, C., Simchowitz, M., Gadipudi, S., and Gupta, A.
\newblock Repo: Resilient model-based reinforcement learning by regularizing posterior predictability.
\newblock \emph{CoRR}, abs/2309.00082, 2023.
\newblock \doi{10.48550/ARXIV.2309.00082}.
\newblock URL \url{https://doi.org/10.48550/arXiv.2309.00082}.

\end{thebibliography}
\bibliographystyle{icml2024}

\newpage
\appendix
\onecolumn
\section{Hyperparameters}
We present all hyperparameters in Table ~\ref{tab:hparams}. 

\begin{table}[h!]
  \caption{Our model's hyperparameters, which are the same across all tasks in DMControl and Realistic Maniskill.
  }
  \centering
  \begin{tabular}{lll}
  \toprule
  \textbf{Name} & \textbf{Symbol} & \textbf{Value} \\
  \midrule
  \multicolumn{3}{l}{\textbf{General}} \\
  \midrule
  Replay capacity (FIFO) & --- & $10^6\!\!$ \\
  Batch size & $B$ & 16 \\
  Batch length & $T$ & 64 \\
  Activation & --- & $\operatorname{LayerNorm}+\operatorname{SiLU}$ \\
  \midrule
  \multicolumn{3}{l}{\textbf{World Model}} \\
  \midrule
  Number of latents & --- & 32 \\
  Classes per latent & --- & 32 \\
  Learning rate & --- & $10^{-4}$ \\
  Adam epsilon & $\epsilon$ & $10^{-8}$ \\
  Gradient clipping & --- & 1000 \\
  \midrule
  \multicolumn{3}{l}{\textbf{Actor Critic}} \\
  \midrule
  Imagination horizon & $H$ & 15 \\
  Discount horizon & $1/(1-\gamma)$ & 333 \\
  Return lambda & $\lambda$ & 0.95 \\
  Critic EMA decay & --- & 0.98 \\
  Critic EMA regularizer & --- & 1 \\
  Return normalization scale & $S$ & $\operatorname{Per}(R, 95) - \operatorname{Per}(R, 5)$ \\
  Return normalization limit & $L$ & 1 \\
  Return normalization decay & --- & 0.99 \\
  Actor entropy scale & $\eta$ & $3\cdot10^{-4}$ \\
  Learning rate & --- & $3\cdot10^{-5}$ \\
  Adam epsilon & $\epsilon$ & $10^{-5}$ \\
  Gradient clipping & --- & 100 \\
  \midrule
  \multicolumn{3}{l}{\textbf{Masking}} \\
  \midrule
  Mask ratio & --- & 50\% \\
  Cube spatial size & $h \times w$ & $10 \times 10$ \\
  Cube depth & $k$ & 4 \\
  \bottomrule
  
  \end{tabular}

  \label{tab:hparams}
  \end{table}
  
\section{Analysis and Example}\label{app:proof}

\subsection{Masking Strategy}
Consider an MDP $\mathcal{M}$ as defined in Section~\ref{sec:prelim}, with vectorized state variables $\zeta=[s;\xi]$, where $\xi=[\xi^0 \xi^1 ... \xi^n]$ is a $n$-dim vector. We begin with an ideal assumption that our masking strategy only applies on exogenous noise $\xi$, \textit{i.e.}, $\tilde{\xi}\subseteq \xi$ be an arbitrary subset (a mask) and $\bar{\xi}=\xi\backslash\tilde{\xi}$ be the variables not included in the mask. Then the state reduces to $\bar{\zeta}=[s; \bar{\xi}]$. And we would like to know if the policy $\bar{\pi}$ under reduced MDP $\bar{\mathcal{M}}$ still being optimal for original MDP $\mathcal{M}$.
\begin{theorem}
If (1) $r(s_t,\xi^i_t,a_t)=0 \forall \xi^i \in \bar{\xi}$, (2) $P(s_{t+1}|s_{t},\xi,a_{t})=P(s_{t+1}|s_{t},\tilde{\xi},a_{t})$, and (3) $P(\tilde{\xi}_{t+1},\bar{\xi}_{t+1}|\tilde{\xi}_t,\bar{\xi}_t)=P(\tilde{\xi}_{t+1}|\tilde{\xi}_t)\cdot P(\bar{\xi}_{t+1}|\bar{\xi}_t)$, then we have $\bar{V}_{\bar{\pi}}(\bar{\zeta})=V_{\bar{\pi}}(\zeta) \forall \zeta \in \mathcal{Z}$, where $\bar{V}_{\bar{\pi}}(\bar{\zeta})$ is the value function under reduced MDP. If $\bar{\pi}$ is optimal for $\bar{\mathcal{M}}$, then $\bar{V}_{\bar{\pi}}(\bar{\zeta})=V^*(\zeta) \forall \zeta\in\mathcal{Z}$.
\end{theorem}
\begin{proof}
    This proof mimics the proof of Theorem 1 in ~\cite{chitnis2020learning}. Consider an arbitrary state $\zeta\in\mathcal{Z}$, and its reduced state $\bar{\zeta}$, we have the following equations:
    \begin{equation}
        V_{\bar{\pi}}(\zeta)=R(\zeta,\bar{\pi}(\bar{\zeta})) + \gamma\sum_{\zeta'}P(\zeta'|\zeta,\bar{\pi}(\bar{\zeta}))\cdot V_{\bar{\pi}}(\zeta').
    \end{equation}
    \begin{equation}
        \bar{V}_{\bar{\pi}}(\bar{\zeta})=R(\bar{\zeta},\bar{\pi}(\bar{\zeta})) + \gamma\sum_{\bar{\zeta}'}P(\bar{\zeta}'|\bar{\zeta},\bar{\pi}(\bar{\zeta}))\cdot V_{\bar{\pi}}(\bar{\zeta}').
    \end{equation}
    Now suppose $V_{\bar{\pi}}^k(\zeta)=\bar{V}_{\bar{\pi}}^k(\bar{\zeta}) \forall \zeta\in\mathcal{Z}$, for some $k$. 
    \begin{align*}
  V^{k+1}_{\bar{\pi}}(\zeta) &= R(\zeta, \bar{\pi}(\bar{\zeta})) + \gamma \sum_{\zeta'} P(\zeta' \mid \zeta, \bar{\pi}(\bar{\zeta})) \cdot V^k_{\bar{\pi}}(\zeta')
  \\&= R(s,\bar{\pi}(\bar{\zeta})) +\sum_{i=1}^n R^i(s, \xi^i, \bar{\pi}(\bar{\zeta})) + \gamma \sum_{\zeta'} P(s' \mid s, \bar{\pi}(\bar{\zeta}), \xi) \cdot P(\xi' \mid \xi) \cdot V^k_{\bar{\pi}}(\zeta') 
  \\&= R(\bar{\zeta}, \bar{\pi}(\bar{\zeta})) + \gamma \sum_{\zeta'} P(s' \mid s, \bar{\pi}(\bar{\zeta}), \xi) \cdot P(\xi' \mid \xi) \cdot V^k_{\bar{\pi}}(\zeta')
  \\&= R(\bar{\zeta}, \bar{\pi}(\bar{\zeta})) + \gamma \sum_{\zeta'} P(s' \mid s, \bar{\pi}(\bar{\zeta}), \bar{\xi}) \cdot P(\xi' \mid \xi) \cdot V^k_{\bar{\pi}}(\zeta')
  \\&= R(\bar{\zeta}, \bar{\pi}(\bar{\zeta})) + \gamma \sum_{s',\bar{\xi}'}\sum_{\tilde{\xi}'} P(s' \mid s, \bar{\pi}(\bar{\zeta}), \bar{\xi}) \cdot P(\tilde{\xi}',\bar{\xi}' \mid \tilde{\xi},\bar{\xi}) \cdot V^k_{\bar{\pi}}(\zeta')
  \\&= R(\bar{\zeta}, \bar{\pi}(\bar{\zeta})) + \gamma \sum_{s',\bar{\xi}'}\sum_{\tilde{\xi}'} P(s' \mid s, \bar{\pi}(\bar{\zeta}), \bar{\xi}) \cdot P(\tilde{\xi}'\mid \tilde{\xi}) P(\bar{\xi}' \mid \bar{\xi}) \cdot V^k_{\bar{\pi}}(\zeta')
    \\&= R(\bar{\zeta}, \bar{\pi}(\bar{\zeta})) + \gamma \sum_{s',\bar{\xi}'}\sum_{\tilde{\xi}'} P(s' \mid s, \bar{\pi}(\bar{\zeta}), \bar{\xi}) \cdot P(\tilde{\xi}'\mid \tilde{\xi}) P(\bar{\xi}' \mid \bar{\xi}) \cdot V^k_{\bar{\pi}}(s',\bar{\xi}')
    \\&= R(\bar{\zeta}, \bar{\pi}(\bar{\zeta})) + \gamma \sum_{\bar{\zeta}}P(\bar{\zeta}' \mid \bar{\zeta}, \bar{\pi}(\bar{\zeta})) \cdot V^k_{\bar{\pi}}(\bar{\zeta}').
  \\&= \bar{V}^{k+1}_{\bar{\pi}}(\bar{\zeta}).
\end{align*}
 Therefore, we have that $\bar{V}_{\bar{\pi}}(\bar{\zeta})=V_{\bar{\pi}}(\zeta) \forall \zeta \in \mathcal{Z}$. And if $\bar{\pi}$ is optimal for $\bar{\mathcal{M}}$, then it is optimal for the full MDP $\mathcal{M}$ as well.
\end{proof}

\subsection{Bisimulation Principle}
\label{app:bisimulation}
Bisimulation measures equivalence relations on MDPs in a recursive manner: two states are considered equivalent if they share equivalent distributions over the next equivalent states and have the same immediate reward \cite{DBLP:conf/popl/LarsenS89, DBLP:journals/ai/GivanDG03}. 
\begin{definition}
Given an MDP $\mathcal{M}$, an equivalence relation $E \subseteq \mathcal{S} \times \mathcal{S}$ is a bisimulation relation if whenever $(s,u)\in E$ the following properties hold, where $\mathcal{S}_E$ is the state space $\mathcal{S}$ partitioned into equivalence classes defined by $E$:
\begin{enumerate}
\item $\forall a \in \mathcal{A}, \mathcal{R}(s, a)=\mathcal{R}(u, a)$
\item $\forall a \in \mathcal{A}, \forall c \in \mathcal{S}_E, \mathcal{P}(s, a)(c)=\mathcal{P}(u, a)(c)$ where $\mathcal{P}(s, a)(c)=\sum_{s^{\prime} \in c} \mathcal{P}(s, a)\left(s^{\prime}\right)$.
\end{enumerate}
Two states $s,u\in\mathcal{S}$ are bisimilar if there exists a bisimulation relation $E$ such that $(s,u)\in E$. We denote the largest bisimulation relation as $\sim$.
\end{definition}
However, bisimulation, by considering equivalence for all actions including bad ones, often leads to "pessimistic" outcomes. To address this, \cite{castro2020scalable} introduced $\pi$-bisimulation, which eliminates the need to consider every action and instead focuses on actions induced by a policy $\pi$.
\begin{definition}\cite{castro2020scalable}
Given an MDP $\mathcal{M}$, an equivalence relation $E^{\pi} \subseteq \mathcal{S} \times \mathcal{S}$ is a $\pi$-bisimulation relation if the following properties hold whenever $(s,u)\in E^{\pi}$:
\begin{enumerate}
\item $r(s,{\pi})=r(u,{\pi})$
\item $\forall C \in \mathcal{S}_{E^{\pi}}, T(C|s,{\pi})=T(C|u,{\pi})$
\end{enumerate}
where $\mathcal{S}_{E^{\pi}}$ is the state space $\mathcal{S}$ partitioned into equivalence classes defined by $E^{\pi}$. Two states $s,u \in S$ are $\pi$-bisimilar if there exists a $\pi$-bisimulation relation $E^{\pi}$ such that $(s,u) \in E^{\pi}$. 
\end{definition}
However, $\pi$-bisimulation is still too strict to be practically applied at scale, as it treats equivalence as a binary property: either two states are equivalent or not, making it highly sensitive to perturbations in numerical values of model parameters. This issue becomes even more pronounced when deep frameworks are employed.
To address this, \cite{castro2020scalable} further proposed a $\pi$-bisimulation metric that incorporates the absolute difference between immediate rewards of two states and the $1$-Wasserstein distance ($\mathcal{W}_1$) between the transition distributions conditioned on the two states and the policy $\pi$:

\begin{theorem}[~\cite{castro2020scalable}]
    Define $\mathcal{F}^{\pi}:\mathcal{M}\rightarrow\mathcal{M} $ by $\mathcal{F}^{\pi}(d)(u,v)=|R(u,\pi)-R(v,\pi)|+\gamma\mathcal{W}_1(d)(P_u^\pi,P_v^\pi)$, then $\mathcal{F}^\pi$ has a least fixed point $d^\pi_\sim$, and $d^\pi_\sim$ is a $\pi$-bisimulation metric.
\end{theorem}
It suffices to show that above fixed-point updates are contraction mappings. Then the existence of a unique metric can be proved by invoke the Banach fixed-point theorem~\cite{ferns2011bisimulation}. An essential assumption is that the state space $\mathcal{S}$ should be compact\footnote{A continuous space is compact if and only if it is totally bounded and complete.}. And the compactness of $\mathcal{S}$ implies that the metric space over this state space is complete such that the Banach fixed-point theorem can be applied. And when considering the approximate dynamics, the situation becomes more complicated. ~\cite{kemertas2021towards} show that:
\begin{theorem}[~\cite{kemertas2021towards}]
    Assume $\mathcal{S}$ is compact. For $d^\pi$, if the support of an approximate dynamics model $\hat{P}$, $\operatorname{supp}(\hat{P})$ is a closed subset of $\mathcal{S}$, then there exists a unique fixed-point $d^\pi$, and this metric is bounded:
\begin{equation}
        \operatorname{supp}(\hat{P}) \subseteq \mathcal{S} \Rightarrow \operatorname{diam}\left(\mathcal{S} ; {d}^\pi\right) \leq \frac{1}{1-\gamma}\left(R_{\max }-R_{\min }\right)
    \end{equation}
\end{theorem}
\begin{proof}
    The proof adapts from~\cite{kemertas2021towards}, which is also a slight generalization of the distance bounds given in Theorem 3.12 of~\cite{ferns2011bisimulation}.
    \begin{equation}
        d^\pi(u,v)=|R(u,\pi)-R(v,\pi)|+\gamma W(d)(P(\cdot|u,\pi),P(\cdot|v,\pi))\leq R_\mathrm{max}-R_\mathrm{min}+\gamma \mathrm{diam}(\mathcal{S};d^\pi), \forall (u,v)\in \mathcal{S}\times\mathcal{S},
    \end{equation}
    with the use of Lemma 5 in~\cite{kemertas2021towards}, we have:
    \begin{equation}
    \begin{aligned}
        \mathrm{diam}(\mathcal{S};d^\pi) &\leq R_\mathrm{max}-R_\mathrm{min}+\gamma \mathrm{diam}(\mathcal{S};d^\pi)\\&\leq \frac{1}{1-\gamma}(R_\mathrm{max}-R_\mathrm{min})
    \end{aligned}
    \end{equation}
\end{proof}
In this paper, our bisimulation objective is defined as follows:
    \begin{equation}
\mathcal{F}^\pi d(s_i,s_j) = |r_{s_i}^\pi-r_{s_j}^\pi|+\gamma E_{\begin{subarray} {l}s_{i+1}\sim \hat{P}_{s_{i}}^a, \\s_{j+1}\sim \hat{P}_{s_{j}}^a\end{subarray}}[d(s_{i+1},s_{j+1})],
\end{equation}
where we sample the next state pairs from an approximated dynamics model RSSM instead of the ground-truth dynamics, and use the independent coupling instead of computing Wasserstein distance. In principle, iteration on conventional state space is acceptable with such a method. While in practice, the above requirement is hard to be satisfied as we learn state representation from an noisy observation space that includes unpredictable exogenous noise. 

\subsection{Counterexample}

Consider two vectorized states $u=(1, 2, 3, 1, 1), v=(2, 1, 1, 1, 1)$, where the last two dimension of these states are exogenous noise that irrelevant to the task. Under policy $\pi$, their next states are $u'=(2, 2, 1, 1, 1), v'=(1, 1, 2, 1, 1)$ respectively. Give $\gamma=0.92, r_u^\pi=0.03, r_v^\pi=0.02$, and with an error of $\epsilon=0.01$, we almost reach the optimal bisimulation distance:
\begin{equation}
\begin{aligned}
    d(u,v)=0.7955\\
    (r_u^\pi-r_v^\pi)+d(u',v')=0.7945\\
    \Delta=0.7955-0.7945=0.001<\epsilon.
\end{aligned}
\end{equation}
Meanwhile, the endogenous states $\bar{u}=(1,2,3), \bar{v}=(2,1,1)$, also achieve their optimal bisimulation distance:
\begin{equation}
\begin{aligned}
    d(\bar{u},\bar{v})=0.7638\\
    (r_u^\pi-r_v^\pi)+d(\bar{u}',\bar{v}')=0.7612\\
    \Delta=0.7638-0.7612=0.0026<\epsilon,
\end{aligned}
\end{equation}
while $u$ and $v$ still contain exogenous noise. That is to say, only with bisimulation principle is sufficient to learn task-relevant information, while not enough to learn compact representation.

\section{Additional Related Work Discussion}
\label{app:comparison}
In this section, we provide an additional related work description, and a detailed comparison between our model and other baselines that developed based on Dreamer, including TIA, Denoised MDP, DreamerPro, and RePo. 

\paragraph{State Representation Learning}
Recent advancements in Reinforcement Learning (RL) emphasize learning state representations to understand environment structures, with successful methods like CURL~\cite{laskin2020curl} and DrQ~\cite{kostrikov2020image, yarats2021mastering} using data augmentation techniques such as cropping and color jittering,  yet their efficacy is closely tied to the specific augmentation employed. Approaches like masking-based approaches~\cite{seo2023multi,yu2022mask, seo2023masked, liu2022masked} aim to reduce spatiotemporal redundancy but often overlook task-relevant information. Bisimulation-based methods~\cite{zhang2020learning, zang2022simsr} focus on learning reward-aware state representations for value-equivalence and sample efficiency, but they face challenges in achieving compact representation spaces since they sample consecutive states from approximated dynamics. Additionally, a branch of research investigates causality to discover causal relationships between state representation and control~\cite{wang2022causal, lamb2022guaranteed, islam2023principled, efroni2021provably, efroni2022sample, fu2021learning, zang2022behavior}. Our work primarily follows the methods based on bisimulation and masking, while developing a hybrid RSSM structure tailored for model-based agents.

\paragraph{TIA~\cite{fu2021learning}} extended Dreamer by creating a cooperative two-player game involving two models: the task model and the distractor model. The distractor model aims to disassociate from the reward as much as possible, while the task model focuses on capturing task-relevant information. Both models contribute to a reconstruction process involving an inferred mask in pixel-space. Although TIA shares similarities with our model, such as the use of masks and a dual-model framework, our hybrid RSSM structure differs in that it does not explicitly model exogenous noise, instead employing a random masking strategy. Moreover, our approach has lower time complexity than TIA, as we utilize a shared historical representation for both branches in the framework, eliminating the need for separate gradient computations. While TIA's learned mask effectively removes noise distractors through pixel-wise composition, it falls short in addressing more general noise types, such as observation disturbances caused by a shaky camera. From this perspective, investigating the potential solution of making masking strategy informed from the control task is still worthful for many approaches inlcuding TIA and ours.

\paragraph{Denoised MDP~\cite{wang2022denoised}}  classified RL information into four types based on controllability and its relevance to rewards, defining useful information as that which is controllable and reward-related. Their approach tends to overlook factors unrelated to control, even if they might influence the reward function. To address this, they introduced a variational mutual information regularizer to separate control and reward-relevant information from overall observations. While this method successfully distinguishes between task-relevant and irrelevant components, Denoised MDP demonstrated higher variance and moderate performance in distraction settings. This may be attributed to its continued reliance on pixel-level reconstruction, which, by focusing on minute details, could unintentionally diminish policy performance in distraction settings. Conversely, our method, eschewing pixel-level reconstruction, flexibly eliminates spatio-temporal redundancies while preserving semantic content, leading to enhanced performance.

\paragraph{DreamerPro~\cite{deng2022dreamerpro}} proposed a reconstruction-free MBRL agent by combining the prototypical representation learning with temporal dynamics learning. Borrowing idea from SwAV~\cite{caron2020unsupervised}, they tried to align the temporal latent state with the cluster assignment of the observation. However, their cluster assignment requires to apply the Sinkhorn Knopp algorithm~\cite{cuturi2013sinkhorn} to update prototypes. This requires more computational cost and more hyperparameters to tune. Besides, DreamerPro still cannot learn task-relevant information as its representation is not informed by reward.

\paragraph{RePo~\cite{DBLP:journals/corr/abs-2309-00082}} developed its representation in a way of maximizing mutual information (MI) between the current representation and all future rewards while minimizing the mutual information between the representation and observation. Excluding pixel-level reconstruction, they ensure latents predictable by optimizing a variational lower bound on the MI-objective which tractably enforces that all components are highly informative of reward. Though being task-specific and compact, RePo is highly sensitive to the hyper-parameters since their objective refer to Lagrangian formulation that includes various factors. Instead, our framework does not rely on hyper-parameter tuning, where we set all parameters fixed for all tasks. This further shows the robustness of our framework.

\section{Experimental Details}
\label{app:experiment}

\subsection{Model Architecture Details}
\label{app:model_detail}

We have developed our model based on DreamerV3, which employs RSSM to learn state representations and dynamics. We fix the input image size to 64 × 64 and use a image encoder which includes a 4-layer CNN with \{32, 64, 128, 256\} channels, a (4, 4) kernel size, a (2, 2) stride. As a result, our embedding size is 4096. 

We implement our dynamics model as a hybrid RSSM, which contains an online RSSM for the mask branch and an EMA RSSM for the raw branch, where the gradients only pass through the online RSSM. The online RSSM is composed of a GRU and MLPs. The GRU, with 512 recurrent units, is used to predict the current mask recurrent state based on the previous mask recurrent state, the previous mask posterior stochastic representation, and the previous action. All stochastic repersentations are sampled from a vector of softmax distributions, and we use straight-through estimator to backpropagate gradients through the sampling operation.
The EMA RSSM has the same structure as the online RSSM.
The size of mask recurrent states $h^m_t$ is 512 and the size of stochastic representations $z^m_t$ and $\hat{z}^m_t$ is 32 × 32. 
The reward predictor, the continue predictor, the transition predictor, the value function, and the actor are all MLPs with two hidden layers, each with 512 hidden units. And we use symlog predictions and the discrete regression approach for the reward predictor and the critic. We use layer normalization and SiLU as the activation function, and update all the parameters using the Adam optimizer.

Notably, despite incorporating dual RSSMs, \textit{i.e.}, mask branch and raw branch, in our framework, our model maintains a slightly smaller overall size (17.54M) compared to the original DreamerV3 (18.22M), which is notable considering the substantial size of the decoder parameters in DreamerV3. Furthermore, our model is also time-efficient due to the removal of the time-cost decoder and the use of a shared historical representation for both branches within the framework.

\subsection{Baselines}
For DreamerV3, we use an unofficial open-sourced pytorch
version of DreamerV3~\cite{dreamerv3torch} as the baseline, and we build our framework on top of it for fair comparison. For RePo and SimSR, we use the official implementation and the reported hyperparameters in their papers. As for other baselines, we simply adopt the data from results reported in RePo.

\subsection{Environment Details}

\paragraph{DMControl tasks with  default settings}
This setting consists several continuous control tasks, wherein the agent solely receives high-dimensional images as inputs. These tasks include \textit{walker\_stand}, where a bipedal agent, referred to as ``walker", is tasked with maintaining an upright position; \textit{walker\_walk} and \textit{walker\_run}, which require the walker to move forward; and \textit{cheetah\_run}, where a bipedal agent, the ``cheetah", is required to run forward rapidly. We also utilized \textit{cartpole\_swingup}, a task involving a pole and cart system where the goal is to swing up and balance the pole; \textit{reacher\_easy}, which involves controlling a two-link reacher, to reach a target location; and \textit{finger\_spin}, where a robotic finger is tasked with continually rotating a body on an unactuated hinge.
We set the time limit to 1000 steps and the action repeat to 2 for all tasks.

\paragraph{DMControl tasks with distraction settings}
To evaluate our model's ability to learn exo-free policy, we test our model in the distraction settings of DMControl. In this setting, we follow DBC~\cite{zhang2020learning} and replace DMControl's simple static background with 1000 frames grayscale videos from the Kinectics-400 Dataset~\cite{kay2017kinetics}, and set the time limit to 1000 steps and the action repeat to 2 for all tasks.

\paragraph{Realistic Maniskill}
This benchmark is based on the Maniskill 2 ~\cite{gu2023maniskill2} environment, which encompasses a variety of tasks for the agent to learn to master human-like manipulation skills. To evaluate our model's ability to learn policy in realistic environments, we follow RePo's setting and use realistic backgrounds from Matterport~\cite{chang2017matterport3d} as distractors. We use 90 scenes from Matterport3D, which are randomly loaded when the environment is reset as distractions for Realistic Maniskill. We set the time limit to 100 steps and the action repeat to 1 for all tasks. We test our method and baselines on the tasks \textit{Lift Cube} and \textit{Turn Faucet}: in \textit{Lift Cube}, the agent is required to elevate a cube beyond a specified height, while in \textit{Turn Faucet}, the agent must turn on the faucet by rotating its handle past a target angular distance.

\section{Additional Experiments}
\label{app:additional_exp}

\subsection{Additional performance comparison}
We present the learning curve of the methods in the default setting in Figure~\ref{app:fig:default_curve}, which is similar to Figure~\ref{fig:defaut_results} while in a different form of visualization.

\begin{figure}[tbp]
  \begin{center}
\centerline{\includegraphics[width=0.7\linewidth]{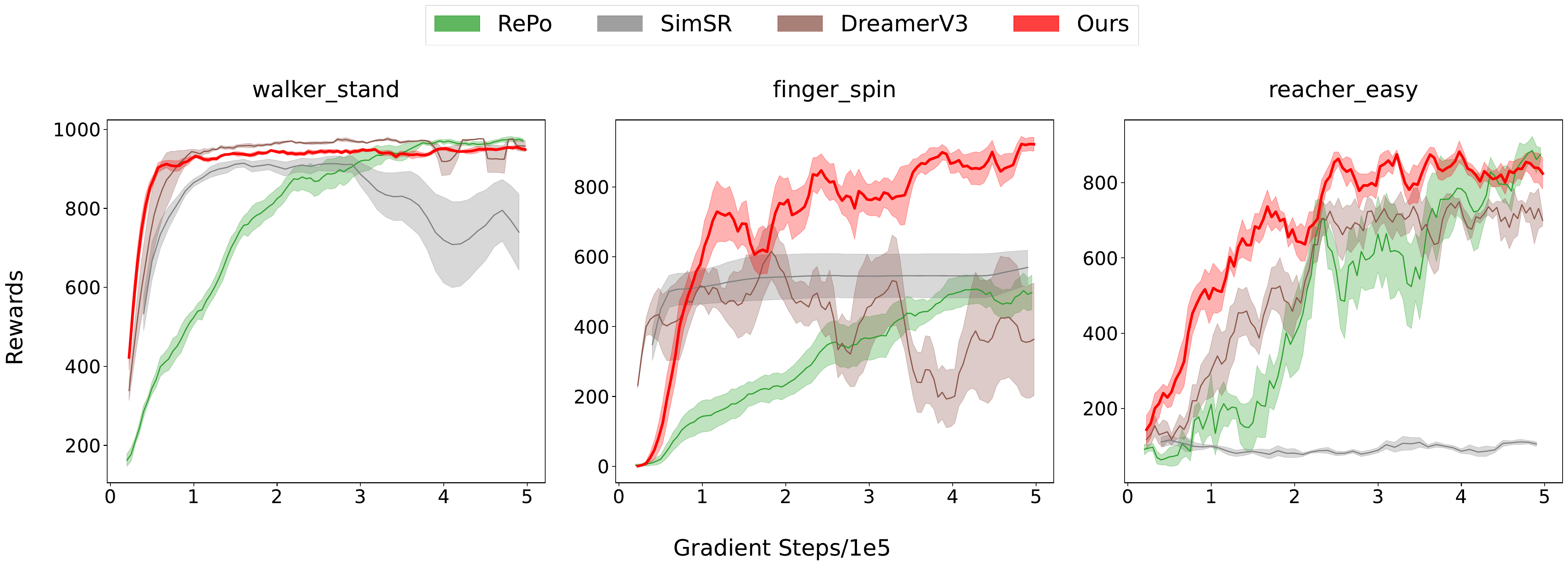}}
  \caption{Performance comparison on DMC tasks over 6 seeds in the default setting.}
  \label{app:fig:default_curve}
  \end{center}
  \end{figure}

\begin{figure}[hp]
    \centering
    \subfigure{
        \includegraphics[width=0.42\textwidth]{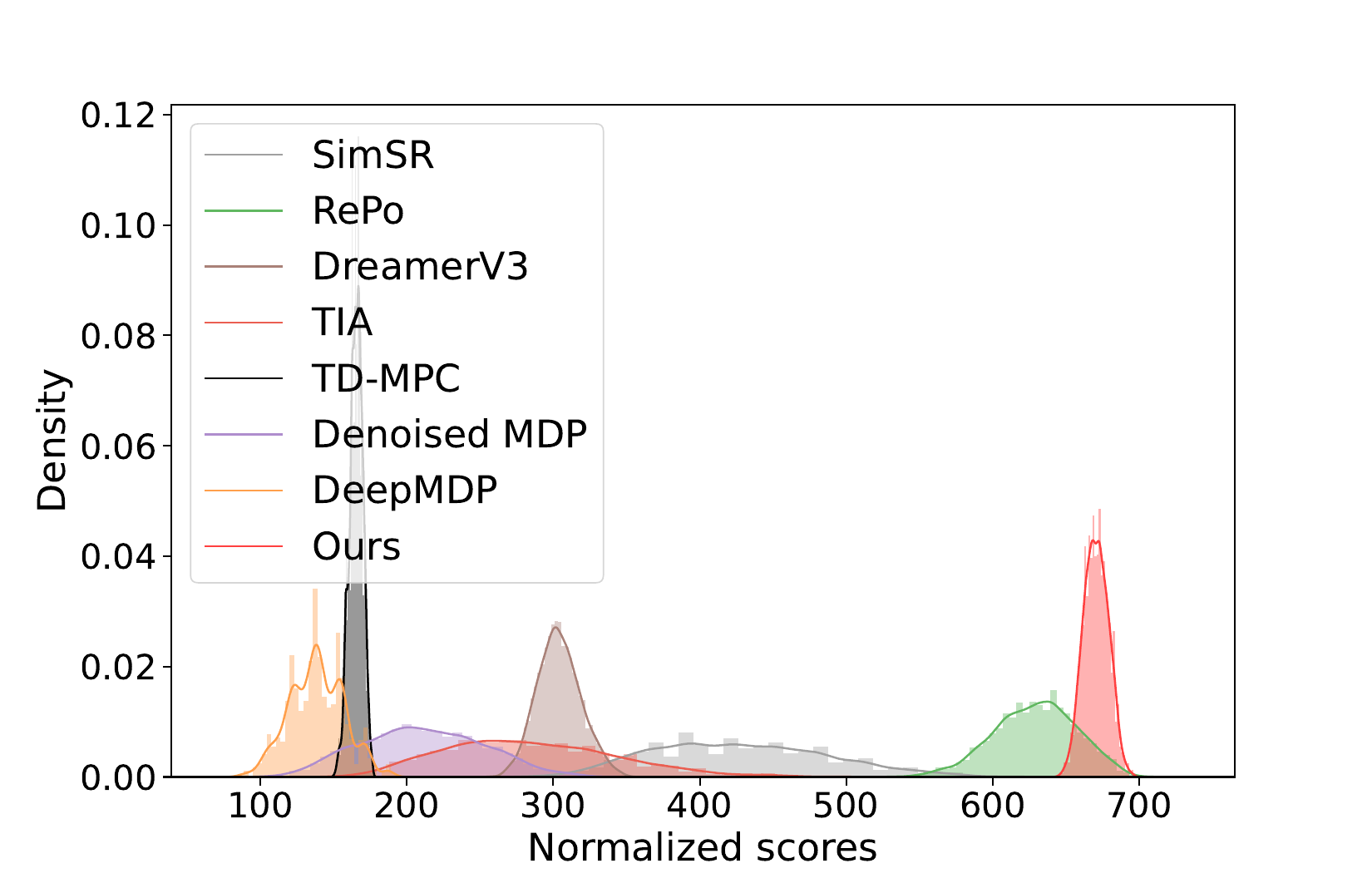}
    }
    \hspace{2mm}
    \subfigure{
        \includegraphics[width=0.42\textwidth]{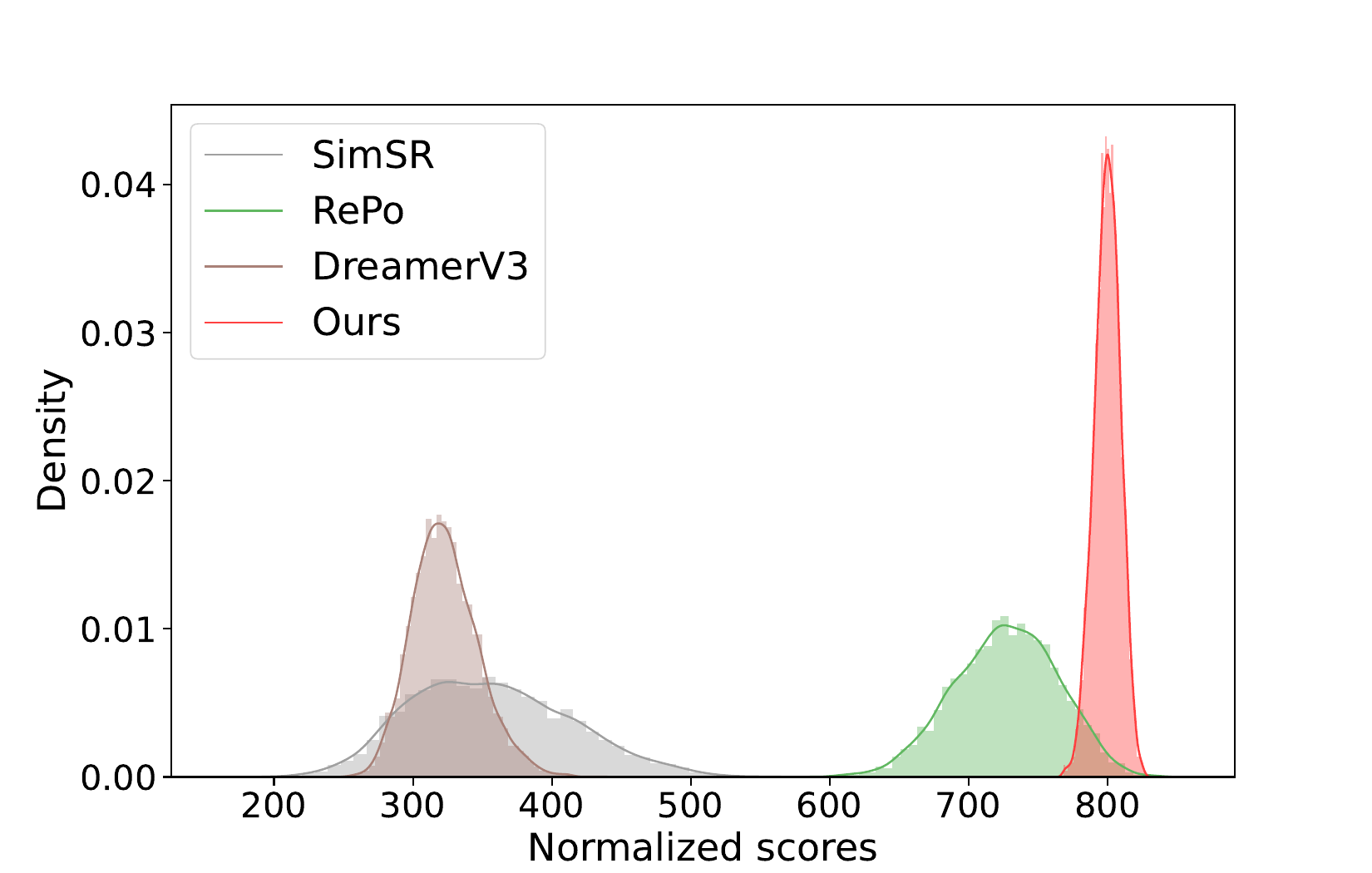}
    }
    \caption{Bootstrapping distributions for uncertainty in IQM (i.e. inter-quartile mean) measurement on DMC tasks in the distraction setting. (\textbf{left}) Averaged on 3 seeds. (\textbf{right}) Averaged on 6 seeds.}
    \label{fig:iqm}
\end{figure}

To further statistically illustrate the effectiveness of our model, we present the bootstrapping distributions for uncertainty in IQM (i.e. inter-quartile mean) measurement on DMC tasks in the distraction setting, following from the performance criterion in~\cite{DBLP:conf/nips/AgarwalSCCB21}. Given that the performance results for certain algorithms have been sourced from \cite{DBLP:journals/corr/abs-2309-00082} and are based on the average across three random seeds, we are unable to calculate the Interquartile Mean (IQM) for all methods with six seeds. Consequently, we present two sets of IQM results in Figure~\ref{fig:iqm}. The first set includes all compared methods and is averaged across three seeds. The second set, which we have derived from re-running and evaluating three representative methods, is based on an average across six seeds, providing us with a more robust statistical measure. The result shows that the final performance of our proposed model is statistically better than all other baselines.

\subsection{Wall clock time comparison}
\label{app:wall_clock}
\begin{wrapfigure}{rt}{0.4\textwidth}

  \begin{center}
    \includegraphics[width=0.5\linewidth]{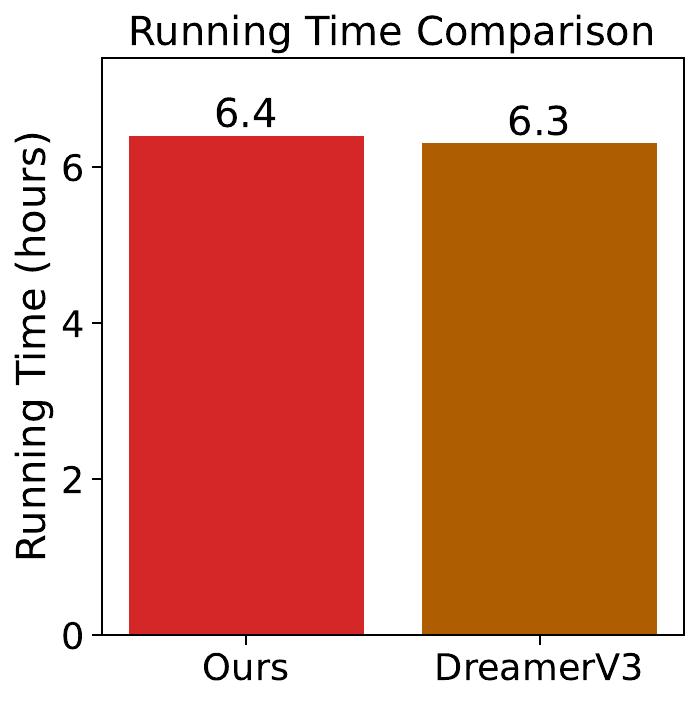}
  \end{center}
  
  \caption{Training Time Comparison on \textit{Lift Cube} task.}
  \label{fig:running_time}

\end{wrapfigure}
We compare the wall-clock traning time of our method and DreamerV3 in the Realistic Maniskill environment, with the use of a sever with NVidia A100SXM4 (40 GB memory) GPU. Figure ~\ref{fig:running_time} shows that the running time of our method almost matches DreamerV3, which represents that our model can achieve significant performance improvements at a lower cost, in the presence of exogenous noise, shows that our method can learn effective representations faster.

\subsection{Ablation studies}

\label{app:ablation}
We evaluate the effectivenss of different components of our model by running the ablation experiments on the DMControl's environment with exogenous noise. All results in this section are averaged across 3 seeds.

\paragraph{Masking-based latent reconstruction and Bisimulation principle}
Our architecture comprises two main components: masking-based latent reconstruction and a similarity-based objective that follows the bisimulation principle. To assess their effectiveness, we conducted ablation studies by excluding each component individually. Specifically, to evaluate the importance of masking-based latent reconstruction, we removed the mask branch, converting our hybrid RSSM back to a standard RSSM and omitting both the cubic masking and the latent reconstruction loss. To assess the bisimulation principle, we removed only the similarity loss while maintaining all other components. 

The results, as shown in Figure \ref{app:fig:ablation_study}, reveal that adding just the similarity-based objective to the DreamerV3 framework does not consistently improve sample efficiency across all tasks. This approach often results in the lowest performance, except in the \textit{cartpole\_swingup} task. In tasks like \textit{reacher\_easy}, the agent fails to develop an acceptable policy, significantly lagging behind in performance compared to other ablations. These findings confirm our theoretical analysis: applying the bisimulation principle directly to model-based agents faces challenges due to the use of an approximate dynamics model for sampling consecutive state representations.

Conversely, utilizing masking-based latent reconstruction generally leads to higher final performance than solely relying on a similarity-based objective. Notably, in nearly half of the tasks, the model with only masking-based latent reconstruction performs comparably to our complete framework, indicating that spatio-temporal information is indeed sparse for these control tasks. Nevertheless, our framework, which includes both components, consistently achieves better performance in most tasks, supporting the necessity of these components. Interestingly, in the \textit{cartpole\_swingup} task, the model with only a similarity-based objective outperforms our full framework, suggesting that the integration of both components is not optimal. A possible explanation is that our masking strategy, which is not selectively applied to exogenous noise but rather uses random masking, might inadvertently impact the endogenous state in some contexts.

\begin{figure*}[ht]
  \begin{center}
  \centerline{\includegraphics[width=0.6\textwidth]{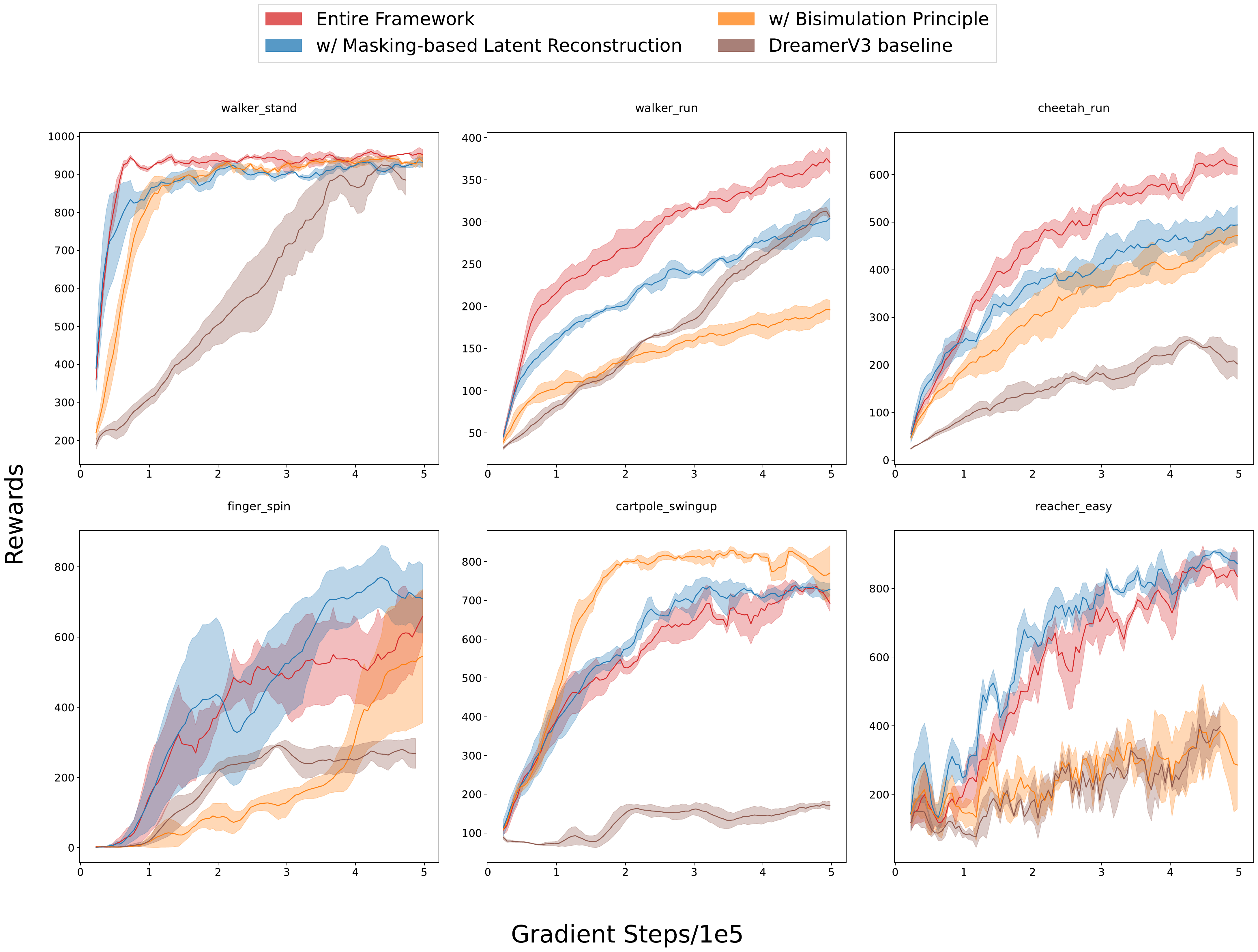}}
  \caption{Results of ablation study on masking-based latent reconstruction and the bisimulation principle.}
  \label{app:fig:ablation_study}
  \end{center}
  \end{figure*}

\paragraph{Normalization for the predictiors}
Our framework incorporates four distinct objectives: latent reconstruction, similarity loss, reward prediction, and episode continuation prediction. For latent reconstruction and similarity loss, we employ normalized state representations because $\ell_2$-normalization ensures that the resulting features are embedded in a unit sphere, which is beneficial for learning state representations. However, the appropriateness of using $\ell_2$-normalization for predicting rewards and episode continuation is not immediately clear. Conventionally, for reward prediction, the exact state representation should be used rather than the normalized one. To investigate this, we conducted an ablation study on the effectiveness of normalization for these two predictors.

The results, illustrated in Figure~\ref{app:fig:ablation_l2_state}, indicate that normalization may introduce unwanted biases into the predictions, leading to a decrease in performance and increased variance. Therefore, we choose un-normalized representation for reward prediction and continuation prediction.

\begin{figure*}[ht]
  \begin{center}
  \centerline{\includegraphics[width=0.6\textwidth]{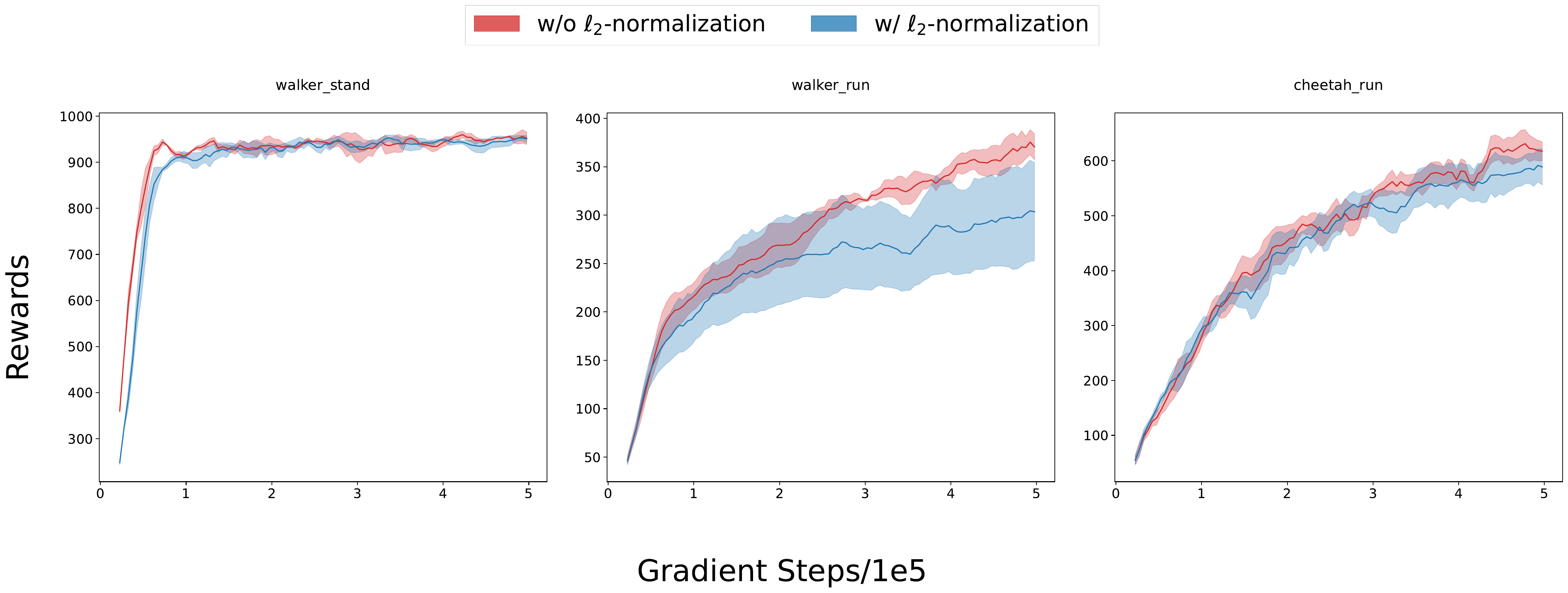}}
  \caption{Results of ablation study on $\ell_2$-Normalizion for predictiors.}
  \label{app:fig:ablation_l2_state}
  \end{center}
  \end{figure*}

\paragraph{Mask-based Similarity Loss}
To integrate the masking strategy with similarity loss, we apply a mask to one state in each pair according to the formula, while keeping the other state unmasked. An alternative approach is to use the masked state representation as the current sample pair and the unmasked ones as the consecutive sample pair, \textit{i.e.},
\begin{equation}
  \begin{aligned}
  \mathcal{L}_\text{sim}&:=\left(d(s_i^m,s_j^m)-\mathcal{F}^\pi d(s_i,s_j)\right)^2\\&=\left(d(s_i^m,s_j^m)-\left(|r_{s_i}^\pi-r_{s_j}^\pi|+\gamma d(\hat{s}_{i+1},\hat{s}_{j+1})\right)\right)^2
  \label{app:eq:another_beh_loss}
  \end{aligned}
  \end{equation}
However, the latter approach may compromise the consistency between the two branches. This is confirmed in Figure ~\ref{app:fig:ablation_simsr_type}, which demonstrates that the first approach is more effective, particularly in tasks like \textit{finger\_spin}. This effectiveness can likely be attributed to the inherent complexity of the task dynamics. The motion of the manipulated object is influenced not only by the actions of the controllable finger but also by the object's intrinsic inertia, as it undergoes rotational motion. This complexity introduces stochasticity and instability into the environment, posing a significant challenge to dynamics modeling and adversely affecting performance, especially as the policy requires forward-looking dynamics modeling. This ablation study underscores the importance of maintaining consistency between the two branches across various tasks.

\begin{figure*}[tbp]
  \begin{center}
  \centerline{\includegraphics[width=0.6\textwidth]{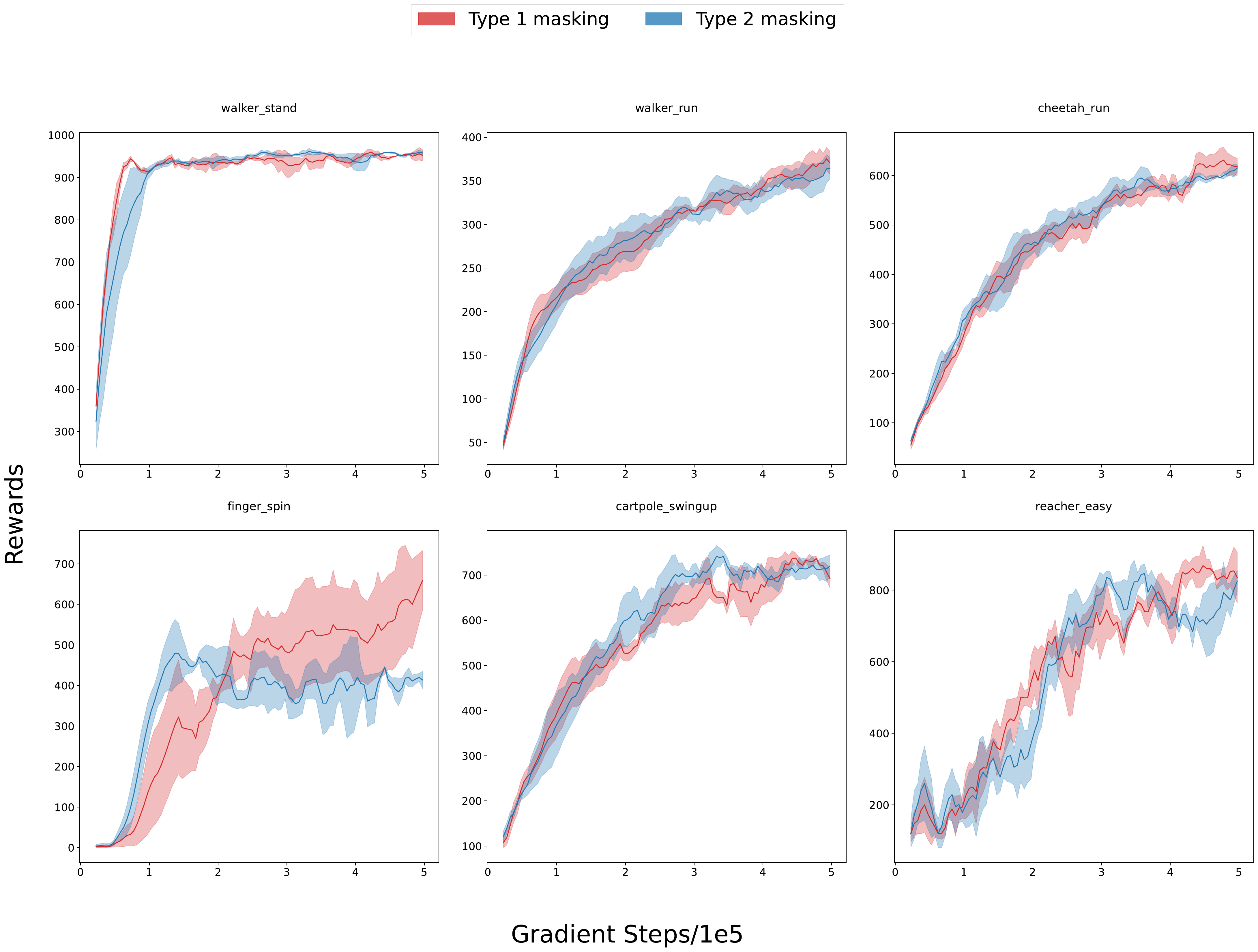}}
  \caption{Results of ablation study on masking strategy for similarity loss.}
  \label{app:fig:ablation_simsr_type}
  \end{center}
  \end{figure*}
  
\begin{figure}[hp]
    \centering
    \subfigure{
        \includegraphics[width=0.3\textwidth]{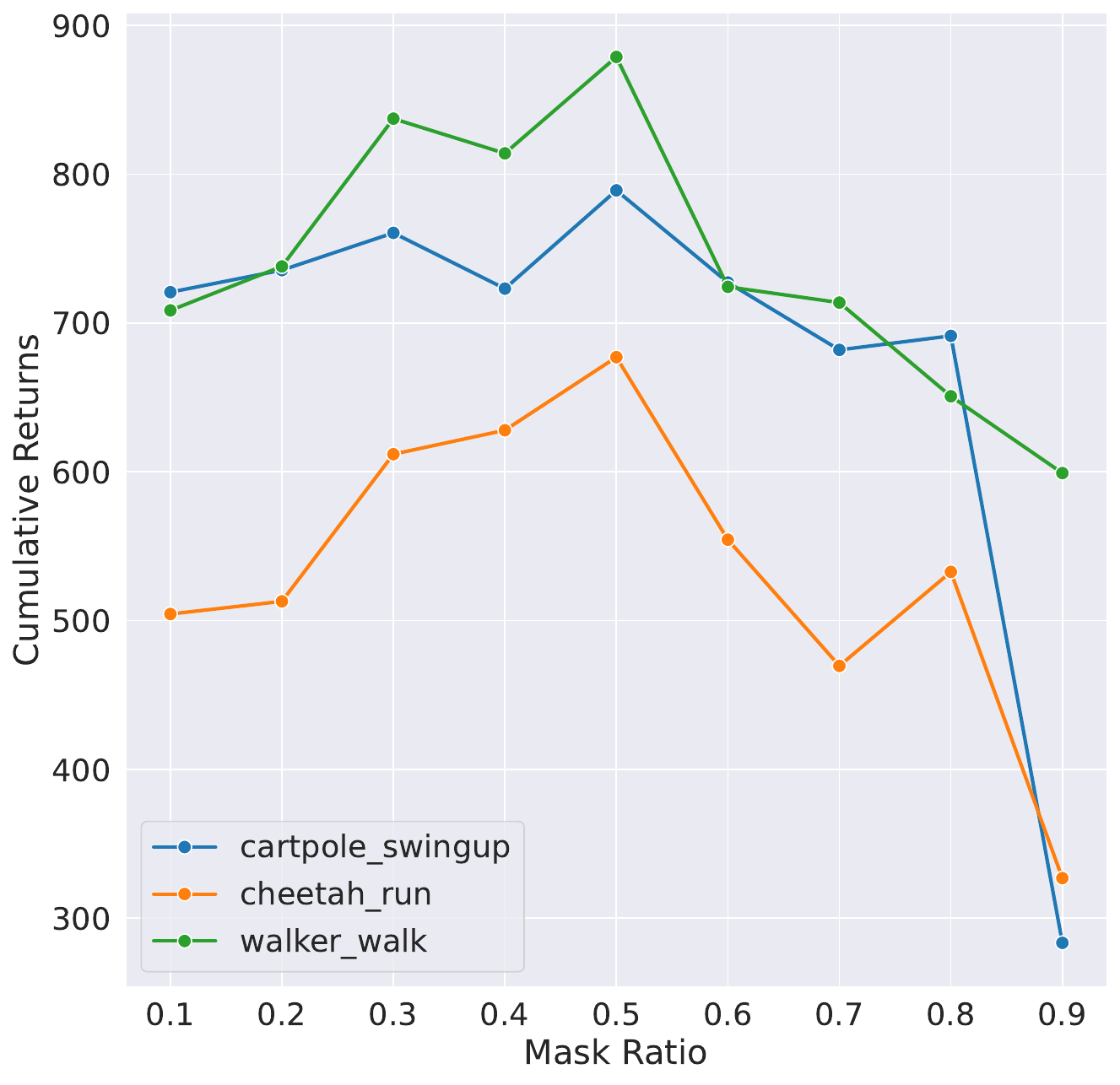}
    }
    \hspace{2mm}
    \subfigure{
        \includegraphics[width=0.3\textwidth]{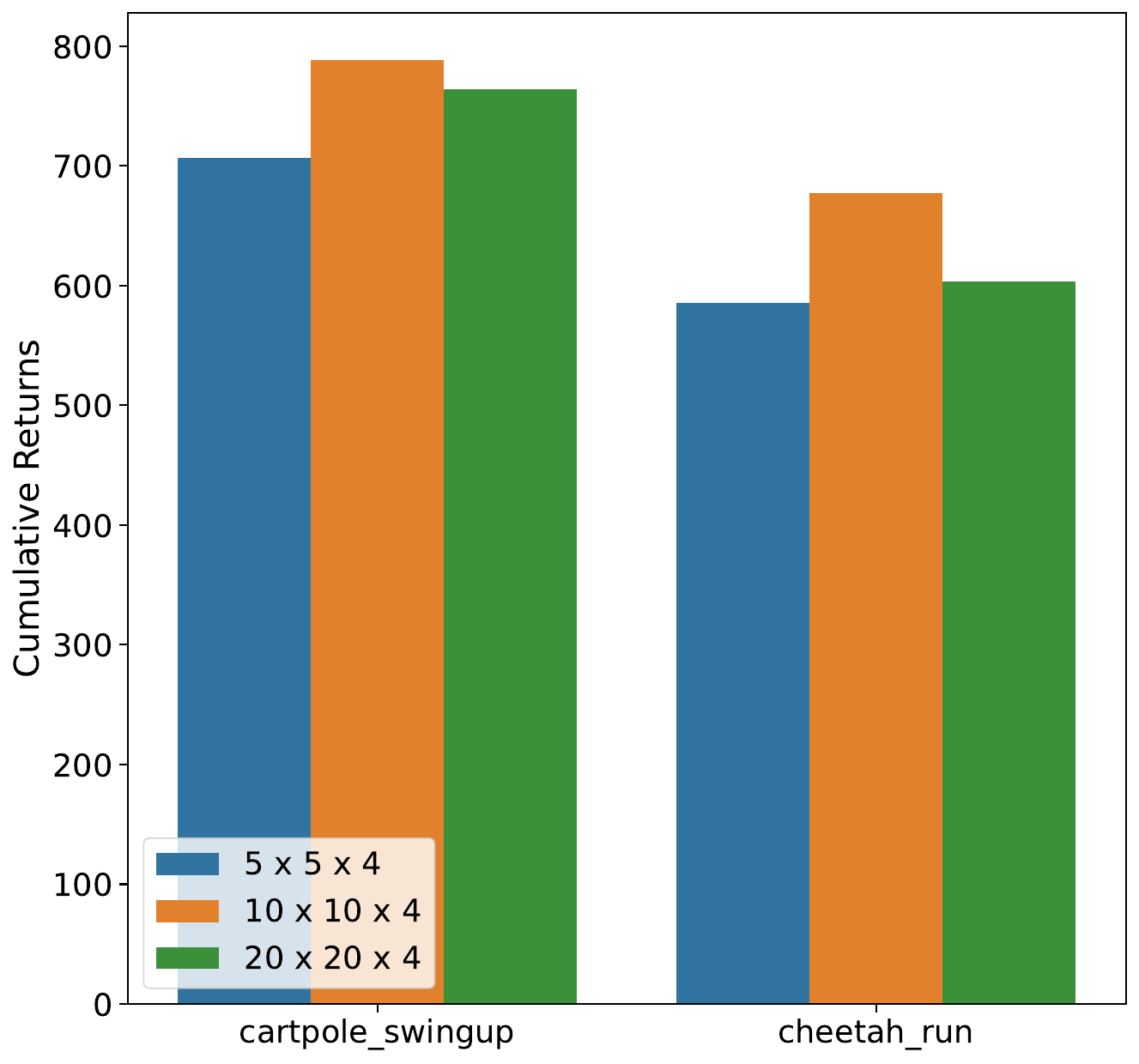}
    }
    \caption{(\textbf{left}) Comparison of different mask ratios in 3 different environments. The final returns are computed at 500k gradient steps updates. (\textbf{Right}) Comparvison of different patch sizes in 2 different environments. The final returns are computed at 500k gradient steps updates.}
    \label{fig:ablation_maskratio_and_patchsize}
\end{figure}
\paragraph{Masking ratio}
We conducted an investigation into the impact of varying mask ratios on the performance of models across different diverse tasks in distraction settings, the result of each task is averaged with three distinct random seeds. The masking ratio is selected within a range of 0.1 to 0.9, with the interval of 0.1. The results  are depicted in Figure~\ref{fig:ablation_maskratio_and_patchsize}.
Contrary to the widely-held assumption that image and video data inherently carry a significant degree of superfluous information, our research indicates that the ideal mask ratio for tasks involving sequential control stands at around 0.5. This is notably lower than the nearly 0.9 mask ratio commonly used in computervision domain, as reported in studies such as \cite{he2022masked} and \cite{videomae}. We believe that this discrepancy can be attributed to the control tasks need to retain more spatiotemporal information than CV tasks to facilitate the sequential control tasks. 


\paragraph{Cuboid Patch Size} 
We also experimented with different cuboid patch sizes , as $(5\times 5\times 4)$, $(10\times 10\times 4)$, and $(20\times 20\times 4)$ respectively. Throughout the experiments, we maintained a masking ratio of 0.5. The results in Figure~\ref{fig:ablation_maskratio_and_patchsize} indicate that the patch size of $10\times 10\times 4$ outperformed both the other two choices. We believe that smaller patch sizes retain unnecessary information, while larger patch sizes may introduce unsuitable masking. Therefore, choosing a moderate patch size is crucial, and in our experiment, we selected $(10\times 10\times 4)$ as the default patch size.

\subsection{Interpretability visualizations}
To verify that our model is indeed capable of filtering task-irrelevant redundancy and learning task-specific features, we utilized the Gradient-weighted Class Activation Mapping (Grad-CAM) technique for feature visualization, as proposed by Selvaraju et al.\cite{selvaraju2017grad}. We generate saliency maps for DMC tasks, and then create a binary map, assigning a value of 1 to pixels in the top 5\% of intensity values, and
0 otherwise, as illustrated in Figure\ref{app:fig:heatmap}.


The results demonstrate that DreamerV3 baseline does not capture any meaningful information that relevant to the task. On contrary, RePo and Our model are proficient at capturing the essential information in images. While RePo struggle to precisely identify the control-relevant parts of the image inputs as it also focus on some irrelevant background noises like the reflections on the surface, our model filter out background noise and focuses on the objects that is crucial for control tasks effectively. These results confirm HRSSM's capability to maintain task-relevant information from visual inputs with exogenous noise.



\begin{figure*}[htbp]
  \begin{center}
  \centerline{\includegraphics[width=0.6\textwidth]{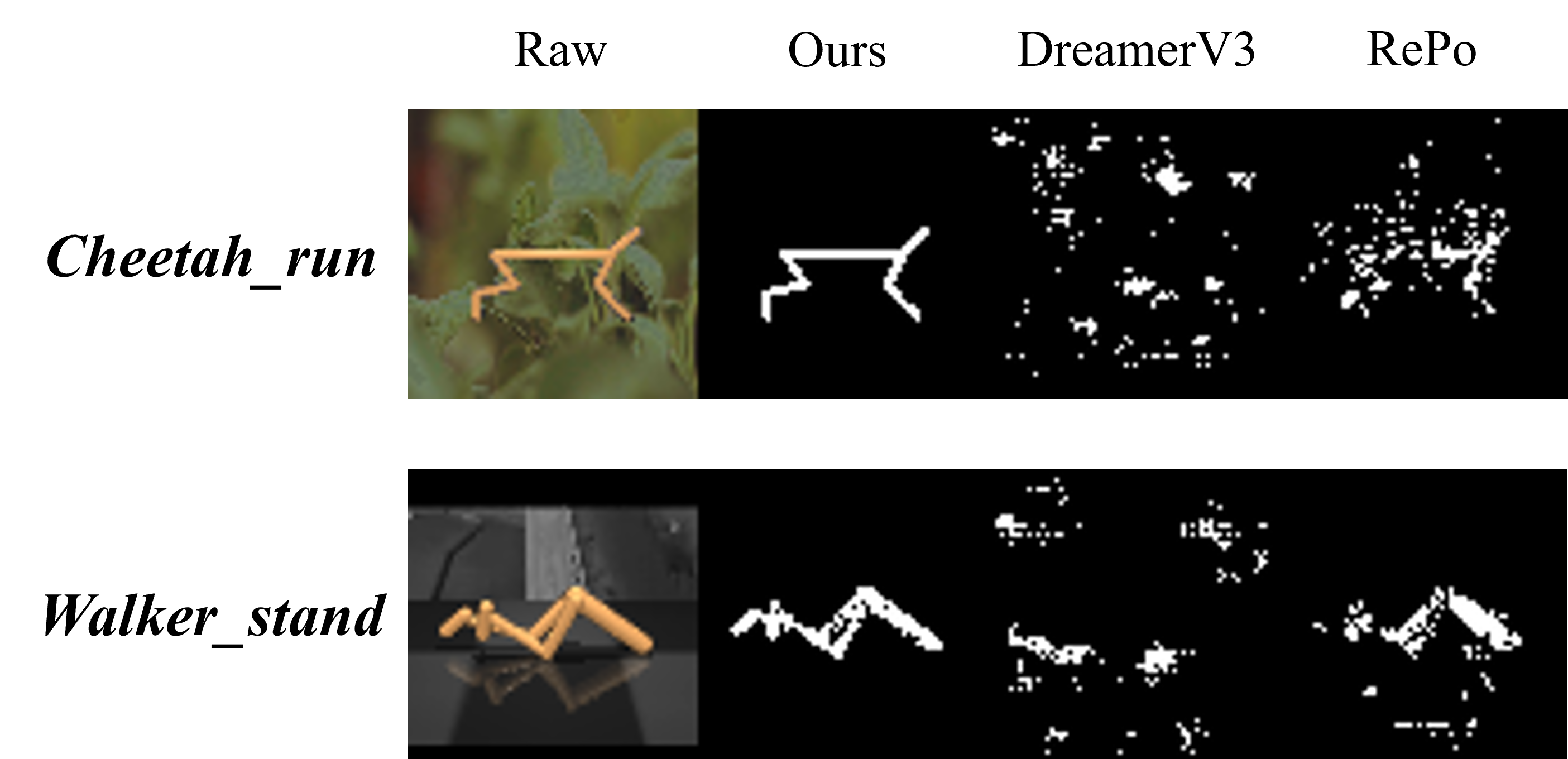}}
  \caption{The feature visualization of the learned representations using Grad-CAM.}
  \label{app:fig:heatmap}
  \end{center}
  \end{figure*}
\subsection{More distractions}
To evaluate the sample-efficiency and generalization ability of our model, we conduct several different distractions with different nature of noise. Specifically, we have nine different distraction types in total, including the ones we benchmarked in main paper. Examples are in Figure~\ref{app:fig:all_distractions}. They are:

\begin{figure*}[htbp]
  \begin{center}
  \centerline{\includegraphics[width=0.9\textwidth]{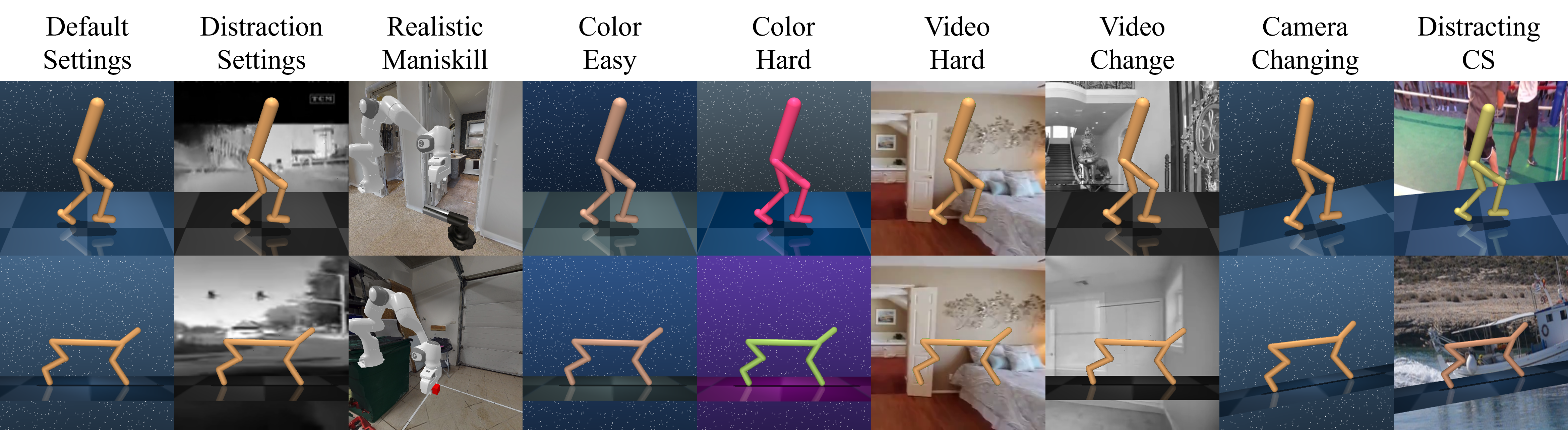}}
  \caption{All kinds of distractions.}
  \label{app:fig:all_distractions}
  \end{center}
  \end{figure*}
  
\begin{table}[ht]
\caption{Overview of Distractions.}
\label{tab:related_works}
\centering
\begin{tabular}{l|c c c c c c c c}
\hline
\makecell{Tasks} & \makecell{Distraction} \\ 
\hline
\makecell{DMC tasks with default settings} & no distraction \\
\makecell{DMC tasks with distraction settings} & video distraction in background \\
\makecell{Realistic Maniskill} & image distraction in background \\
\makecell{Color\_easy in DMC-GS} &  colors slightly change for both agent and background\\
\makecell{Color\_hard in DMC-GS} &  colors dramatically change for both agent and background\\
\makecell{Video\_hard in DMC-GS} &  video distraction in background, the surface is not visible \\
\makecell{Video\_category\_changing} & different set of video distraction in background \\
\makecell{Camera\_changing} & camera positions change\\
\makecell{Distracting\_CS} &  all distractions (color, video, camera)\\
\hline
\end{tabular}
\vspace{-1.em}
\end{table}

\paragraph{DMC tasks with default settings} This is the default setting of DMC tasks without any distractions. It can be seen as the ideal setting in the realistic tasks.
\paragraph{DMC tasks with distraction settings} In this setting, we test the agent in an environment with the background disturbed by the videos from Kinetics dataset~\cite{kay2017kinetics} with the label of driving\_car. During the training and evaluation, the environments are both disturbed by the same category of videos, so it is possible that the agent evaluated on the environments that have been seen.
\paragraph{Realistic Maniskill} Similar to DMC tasks with distraction settings. Further, to simulate real-world scenarios, we replace the default background with realistic scenes from the Habitat Matterport dataset~\cite{ramakrishnan2021habitat}, curating 90 different scenes and randomly loading a new scene at the beginning of each episode. So it can be viewed as image distraction in background.
\paragraph{Color\_easy in DMC-GS} One setting from DeepMind Generalization Benchmark\cite{hansen2021generalization}. We randomize the color of background, floor, and the agent itself, while the colors used are similar to the colors of the original object.
\paragraph{Color\_hard in DMC-GS} One setting from DeepMind Generalization Benchmark\cite{hansen2021generalization}. Similar to Color\_easy, while the colors used is totally different from the colors of the original object.
\paragraph{Video\_hard in DMC-GS} One setting from DeepMind Generalization Benchmark\cite{hansen2021generalization}. Similar to DMC tasks with distraction settings, while the surface is no longer visible.
\paragraph{Video\_category\_changing} A variation of DMC tasks with distraction settings. During the evaluation, we use a totally different category of videos as background, which makes the testing environments all unseen.
\paragraph{Camera\_changing} A variation from Distracting Control Suite (Distracting\_CS)~\cite{DBLP:journals/corr/abs-2101-02722} benchmark. We change the span of camera poses and the camera velocity continually throughout an episode.
\paragraph{Distracting\_CS} Distracting Control Suite (Distracting\_CS)~\cite{DBLP:journals/corr/abs-2101-02722} benchmark is extremely challenging, where camera pose, background, and colors are continually changing throughout an episode. The surface remains visible, such that the agent can orient itself during a changing camera angle.

With the empirical findings presented for the first three distractions in Section~\ref{sec:experiments}, our analysis now focuses on the agents' performance against the remaining six distractions. Due to the time limitation, we only have two baseline algorithms tested in these distraction settings in our comparison: SVEA~\cite{hansen2021stabilizing} and RePo~\cite{DBLP:journals/corr/abs-2309-00082}. SVEA is a model-free framework that enhances stability in Q-value estimation by selectively applying data augmentation, optimizing a modified Bellman equation across augmented and unaugmented data. For RePo, we search several combinations of hyperparameters, and choose the best hyperparameter pair in the DMC tasks with distraction setting as default for each task. All average returns are averaged by 3 different random seeds. Notably, for Video\_category\_changing setting, we directly evaluate the performance of the agent previously trained on DMC tasks with distraction settings, where we provide the agent's final performance metrics rather than a performance progression curve. The results from Figure~\ref{app:fig:color_easy} to Figure~\ref{app:fig:distracting_cs} and Table~\ref{app:tab:video_category_changing} show that, our model consistently achieve the highest final performance and the best sample-efficiency among the most distractions and most tasks, which indicate that our model's robustness and generalization ability across these kinds of distractions.

\begin{figure*}[htbp]
  \begin{center}
  \centerline{\includegraphics[width=0.6\textwidth]{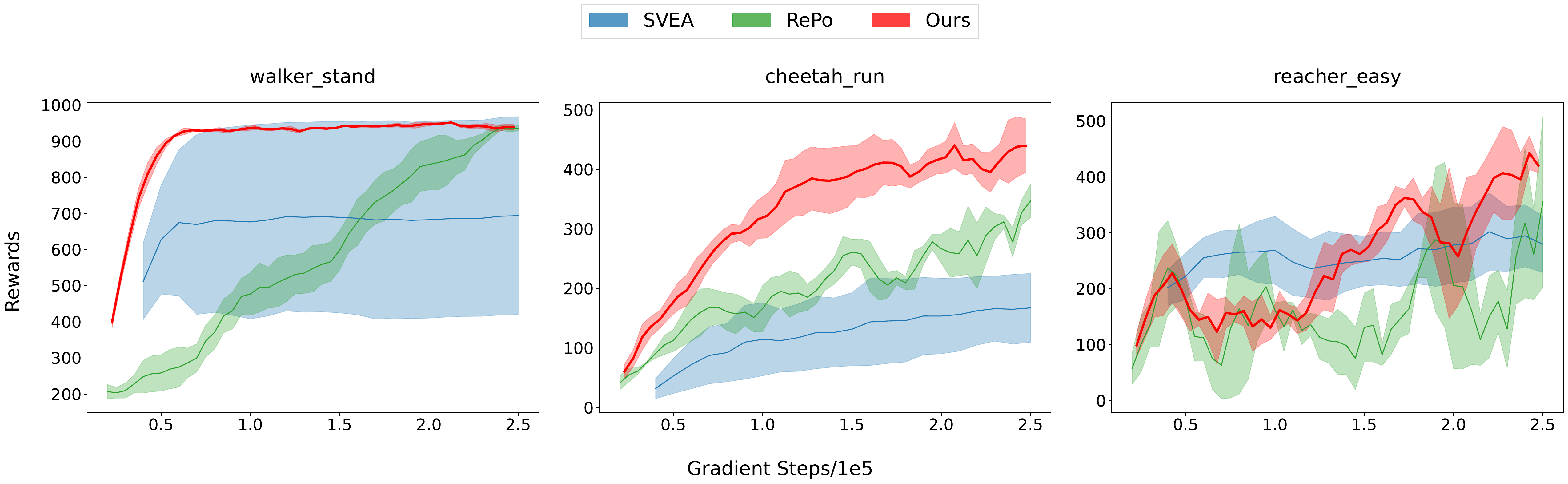}}
  \caption{Performance comparison on Color\_easy in DMC-GS.}
  \label{app:fig:color_easy}
  \end{center}
  \end{figure*}

\begin{figure*}[htbp]
  \begin{center}
  \centerline{\includegraphics[width=0.6\textwidth]{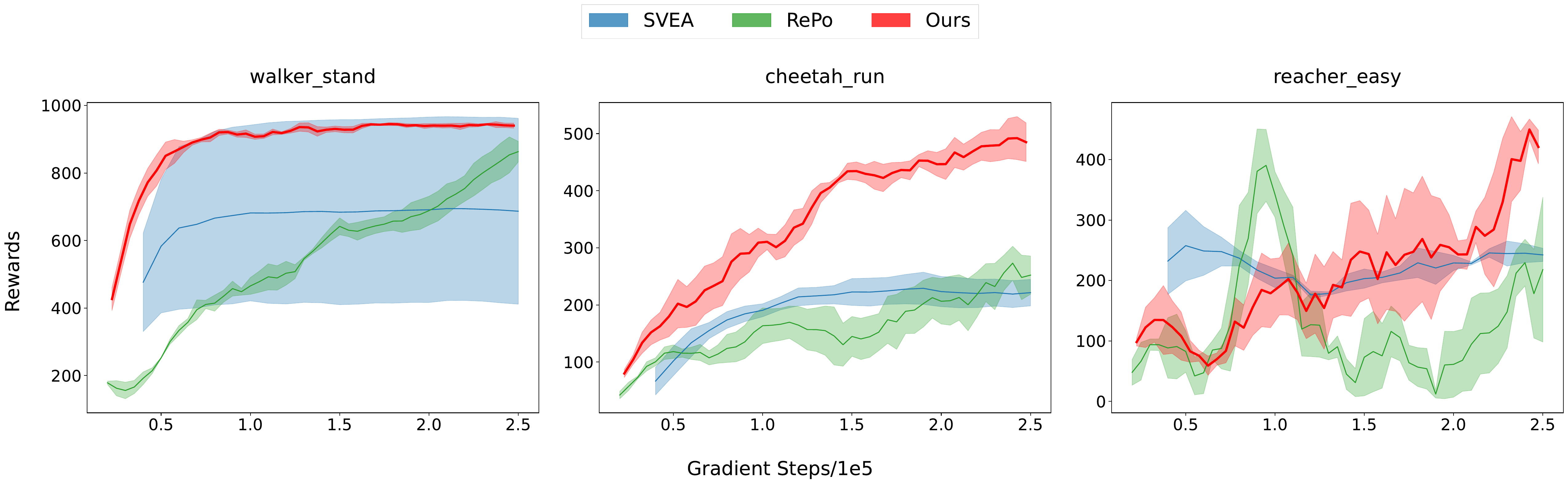}}
  \caption{Performance comparison on Color\_hard in DMC-GS.}
  \label{app:fig:color_hard}
  \end{center}
  \end{figure*}
  
\begin{figure*}[htbp]
  \begin{center}
  \centerline{\includegraphics[width=0.6\textwidth]{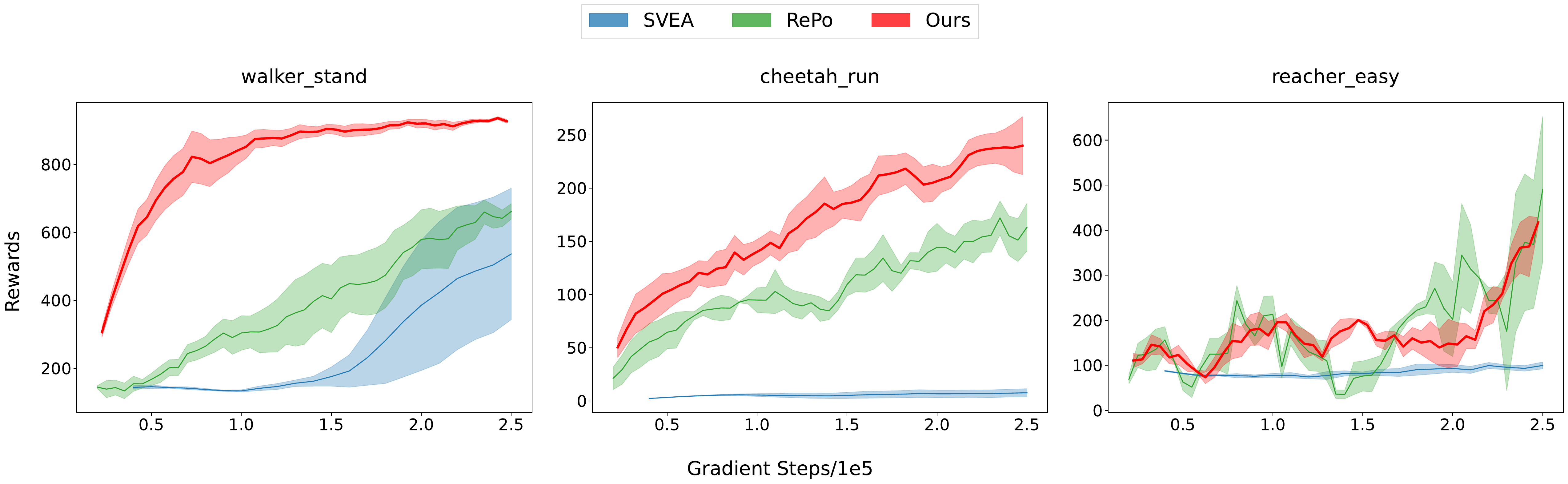}}
  \caption{Performance comparison on Video\_hard in DMC-GS.}
  \label{app:fig:video_hard}
  \end{center}
  \end{figure*}
  
\begin{table}[!htbp]
\caption{Final Performance on Video\_category\_changing. The second column is the final performance on DMC with distraction settings, where the videos come from the same category. The result shows that our model has good generalization ability to adapt to the unseen backgrounds.}
\label{app:tab:video_category_changing}
\centering
\begin{tabular}{l|c c c c c c c c}
\hline
\makecell{Tasks} & \makecell{Rewards on DMC with distraction} & \makecell{Rewards on Video\_category\_changing} \\ 
\hline
\makecell{walker\_stand} & 946 $\pm$ 12 & 963 $\pm$ 11\\
\makecell{walker\_walk} & 877 $\pm$ 35 & 868 $\pm$ 47 \\
\makecell{walker\_run} & 390 $\pm$ 18 & 406 $\pm$ 15\\
\makecell{cheetah\_run} & 652 $\pm$ 47 & 628 $\pm$ 60\\
\makecell{cartpole\_swingup} & 785 $\pm$ 25 & 755 $\pm$ 24 \\
\makecell{reacher\_easy} & 881 $\pm$ 72 & 905 $\pm$ 31 \\
\hline
\end{tabular}
\vspace{-1.em}
\end{table}
  
\begin{figure*}[htbp]
  \begin{center}
  \centerline{\includegraphics[width=0.6\textwidth]{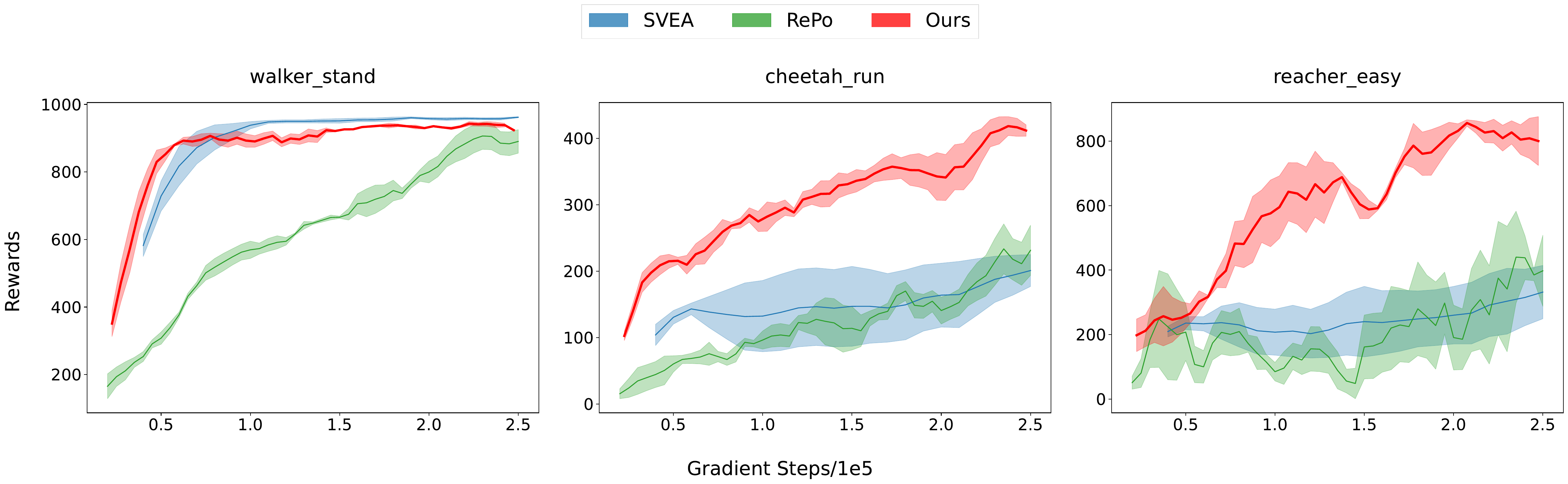}}
  \caption{Performance comparison on Camera\_changing.}
  \label{app:fig:camera_changing}
  \end{center}
  \end{figure*}

\begin{figure*}[htbp]
  \begin{center}
  \centerline{\includegraphics[width=0.6\textwidth]{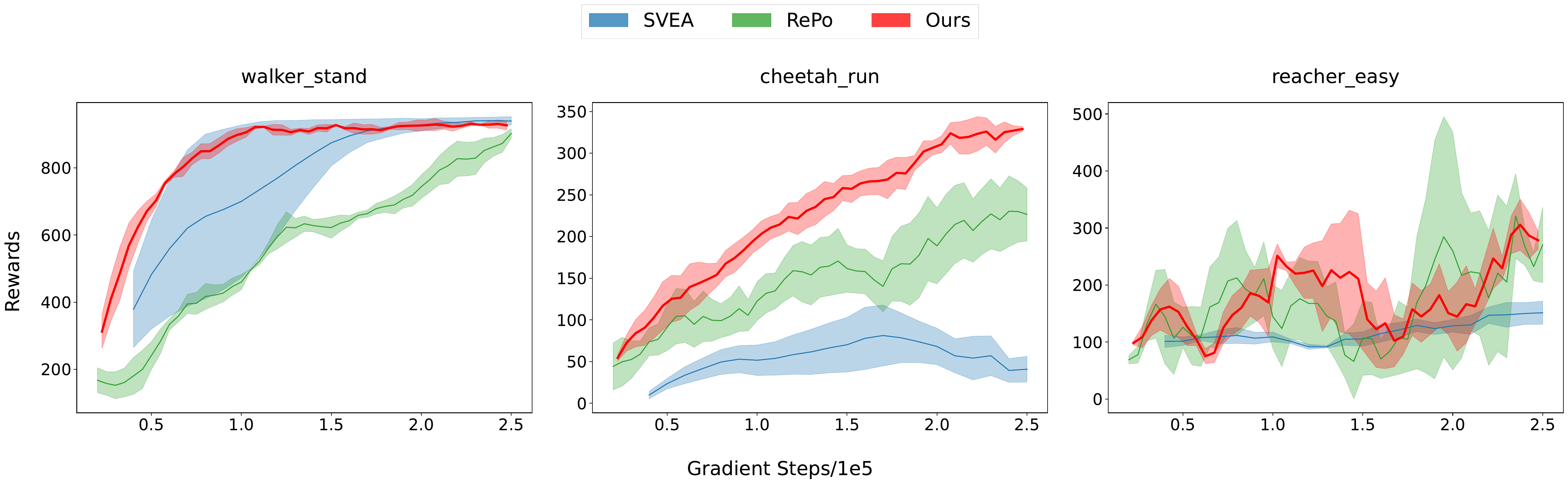}}
  \caption{Performance comparison on Distracting\_CS.}
  \label{app:fig:distracting_cs}
  \end{center}
  \end{figure*}

\subsection{Analysis of Failure Cases and Bottlenecks of HRSSM}
\begin{figure*}[h]
  \begin{center}
  \centerline{\includegraphics[width=0.4\textwidth]{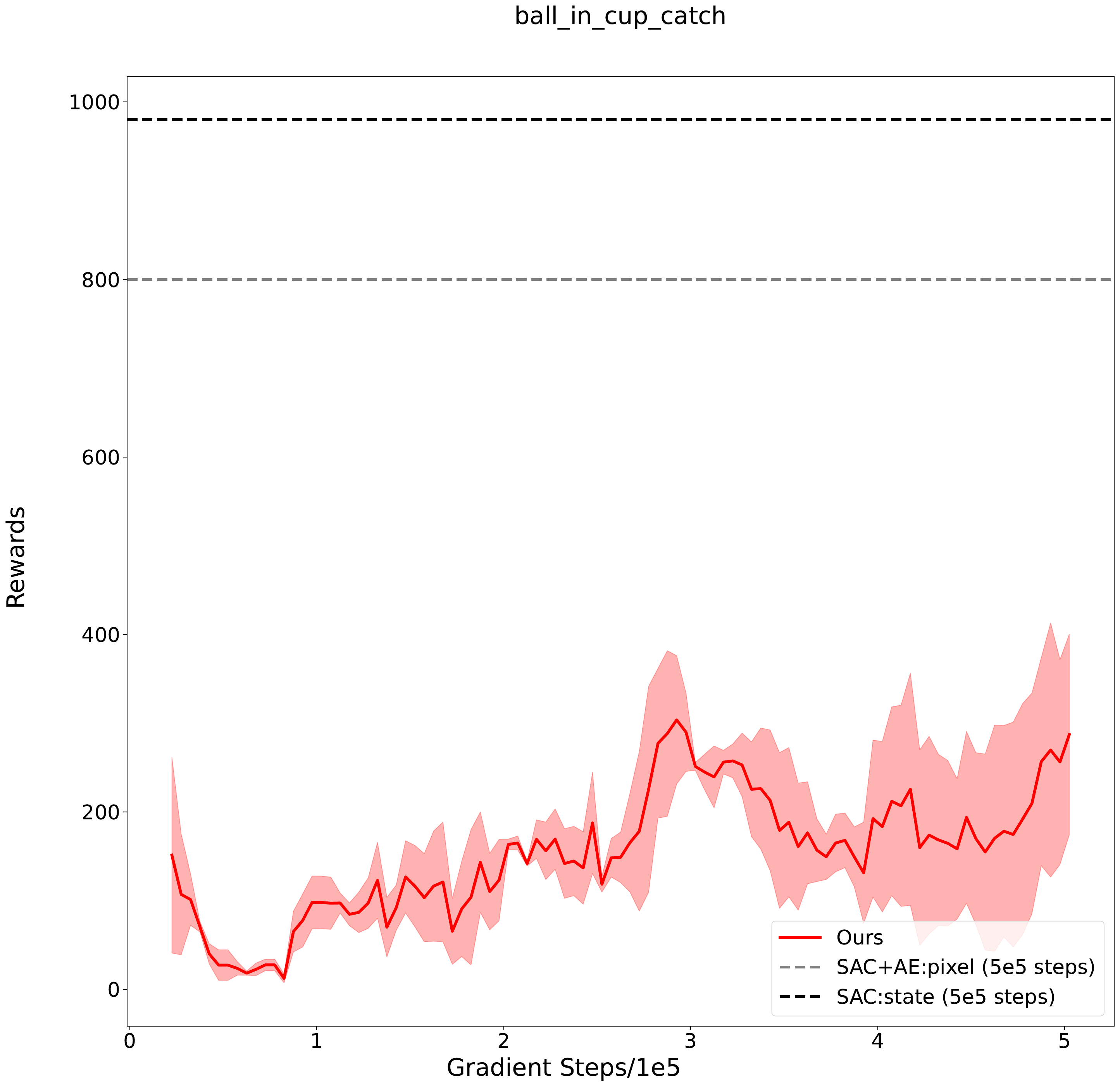}}
  \caption{The performance on ball\_in\_cup\_catch task. The gray dotted line is the final performance of SAC+AE~\cite{yarats2021improving} agent with pixel input (same input as ours) that evaluated at 5e5 gradient steps, and the black dotted line is SAC agent with raw state as input that evaluated at 5e5 gradient steps. In this sparse reward task, the performance of our model is not good as the others.}
  \label{app:fig:bad_case}
  \end{center}
  \end{figure*}
Although our model is effective in most senarios, as illustrated in previous experiments, there still exist cases that our model is not capable of handling well. For instance, the result of \textit{cartpole\_swingup} task in our ablation study show that the final return of the model that only follows bisimulation principle is higher than our entire model, we consider this can be partially attribute to the inappropriate masking. A better masking stategy may help. On the other hand, since our model follows the bisimulation principle, it may fail in sparse reward domains, as the fact that the form of bisimulation computation relies on bootstrapping with respect to the reward function in recursive terms. We present an evaluation on \textit{ball\_in\_cup\_catch} task in Figure~\ref{app:fig:bad_case}. The result substantiates the conclusion that our model does not perform well on this kind of tasks. To address this deficiency, employing goal-conditioned RL techniques or implementing reward re-labeling strategies could potentially offer solutions to improve the performance, we leave this to future work.

\section{Algorithm}

Our traning algorithm is shown in Algorithm~\ref{algorithm:ours}.

\begin{algorithm}[tb]
  \caption{HRSSM}
  \label{algorithm:ours}
  \begin{flushleft}
      \textbf{Require}: The mask encoder $\mathcal{E}_\phi$, the mask posterior model $q_\phi$, the mask recurrent model $f_\phi$, the mask transition predictor $\pp(\hat{z}^m_t | h^m_t)$, their EMA part $\mathcal{E}'_\phi$, $q'_\phi$, $f'_\phi$, $p'_\phi(\hat{z}_t | h_t)$, the reward predictor $\pp(\hat{r}_t | h^m_t,z^m_t)$ and continue predictor $\pp(\hat{c}_t | h^m_t,z^m_t)$, the critic $v_\psi$ and the actor $\pi_\psi$ ; the cube masking function CubeMask(·), the optimizer Optimizer(·, ·).
  \end{flushleft}
  \begin{algorithmic}[1] 
  \STATE Determine the weight of bisimulation loss $\beta$, observation sequence length $K$, mask ratio $\eta$, cube shape $k\times h\times w$, and EMA coefficient $m$.
  \STATE Initialize a replay buffer $\mathcal{D}$.
  \STATE Initialize all parameters.
  \WHILE{train}
  \FOR{update step c = 1...C}
  \STATE // Dynamics learning
  \STATE Sample $B$ data sequences $\{(a_t,o_t,r_t)\}_{t=k}^{k+T-1}$ from replay buffer $\mathcal{D}$\\
  \STATE Cube masking the observation sequence: $\{o^m_t\}_{t=k}^{k+T-1} \leftarrow CubeMask(\{o_t\}_{t=k}^{k+T-1})$
  \STATE Siamese Encoding: $\{e^m_t\}_{t=k}^{k+T-1} \leftarrow \mathcal{E}_\phi(\{o^m_t\}_{t=k}^{k+T-1})$, $\{e_t\}_{t=k}^{k+T-1} \leftarrow \mathcal{E}'_\phi(\{o_t\}_{t=k}^{k+T-1})$\\
  \STATE Compute mask states: $z^m_t \sim \qp(z^m_t | h^m_t,e^m_t)$, $h^m_t = f_\phi(h^m_{t-1},z^m_{t-1},a_{t-1})$, $\hat{z}^m_t \sim \pp(\hat{z}^m_t | h^m_t)$ \\ 
  \STATE Compute true states: $z_t \sim q'_\phi(z_t | h^m_t,e_t)$, $h_t = f'_\phi(h^m_{t-1},z_{t-1},a_{t-1})$, $\hat{z}_t \sim p'_\phi(\hat{z}_t | h_t)$
  \STATE Predict rewards and continuation flags: $\hat{r}_t \sim \pp(\hat{r}_t | h^m_t,z^m_t)$, $\hat{c}_t \sim \pp(\hat{c}_t | h^m_t,z^m_t)$
  \STATE Calculate $\mathcal{L}_{\mathrm{dyn}}$ according to Eq.~\ref{eq:KL loss}
  \STATE Calculate $\mathcal{L_{\mathrm{rec}}}$ according to Eq.~\ref{eq:rec_loss}  
  \STATE Calculate $\mathcal{L}_\text{sim}$ according to Eq.~\ref{eq:beh_loss}
  \STATE Calculate $\mathcal{L_{\mathrm{pred}}}$ according to Eq.~\ref{eq:pred_loss}
  \STATE Calculate total loss $\mathcal{L}(\phi)=\operatorname{E}_{q_{\phi}}\Big{[}\textstyle\sum_{t=1}^{T}\big{(}\mathcal{L}_{\mathrm{dyn}}(\phi)+\mathcal{L}_{\mathrm{rec}}(\phi)+\mathcal{L}_{\mathrm{sim}}(\phi)+\mathcal{L}_{\mathrm{pred}}(\phi)\big{)}\Big{]}$\\
  \STATE Update the encoder's, RSSM's and predictors' parameters:$\mathcal{E}_\phi, q_\phi, f_\phi, p_\phi$ $\leftarrow$ $Optimizer(\mathcal{E}_\phi, q_\phi, f_\phi, p_\phi, \mathcal{L}(\phi))$\\
  \STATE Update the EMA part's parameters:$\mathcal{E}'_\phi \leftarrow m \mathcal{E}_\phi  + (1 - m) \mathcal{E}'_\phi, q'_\phi \leftarrow m q_\phi  + (1 - m) q'_\phi, 
  f'_\phi \leftarrow m f_\phi  + (1 - m) f'_\phi$,\\ $p'_\phi(\hat{z}_t | h_t) \leftarrow m p_\phi(\hat{z}_t | h_t)  + (1 - m) p'_\phi(\hat{z}_t | h_t)$\\
  // Behavior learning
  \STATE Imagine trajectories $\{(s_t,a_t)\}_{t=k}^{k+H-1}$ from each $s_t$.\\
  \STATE Compute rewards $\{r_t\}_{t=k}^{k+H-1}$, continuation flags $\{c_t\}_{t=k}^{k+H-1}$\\
  \STATE Update the actor $\pi_\psi$ and the critic $v_\psi$'s parameters using actor-critic learning.
  
  \ENDFOR
  \STATE Interact with the environment based on the policy\\
  \ENDWHILE
  \end{algorithmic}
  \end{algorithm}

\end{document}